\documentclass{article}

\PassOptionsToPackage{numbers, compress}{natbib}
\usepackage[preprint]{neurips_2021}

\usepackage{microtype}
\usepackage{multirow}
\usepackage{graphicx}
\usepackage{xcolor}
\usepackage{subfig}
\usepackage{float}
\usepackage{textcomp}
\usepackage{booktabs} 
\usepackage{amsfonts, amsmath, amssymb, amsthm, graphicx}
\usepackage{tikz}
\usetikzlibrary{positioning,chains,fit,shapes,calc} 
\usepackage[ruled,vlined]{algorithm2e}
\usepackage[noend]{algcompatible}

\newtheorem{theorem}{Theorem}
\newtheorem{claim}{Claim}

\newtheorem{corollary}{Corollary}

\theoremstyle{definition}
\newtheorem{definition}{Definition}

\newcommand{\cmmnt}[1]{}
\usepackage[utf8]{inputenc} 
\usepackage[T1]{fontenc}    
\usepackage{hyperref}       
\usepackage{url}            
\usepackage{booktabs}       
\usepackage{amsfonts}       
\usepackage{nicefrac}       
\usepackage{microtype}      
\usepackage{xcolor}         

\newcommand\blfootnote[1]{%
  \begingroup
  \renewcommand\thefootnote{}\footnote{#1}%
  \addtocounter{footnote}{-1}%
  \endgroup
}

\definecolor{myblue}{RGB}{0,0,100}
\definecolor{mygreen}{RGB}{0,0,100}
\definecolor{myyellow}{RGB}{0,100,100}
\definecolor{my1}{RGB}{100,0,100}
\definecolor{myred}{RGB}{100,0,0}
\definecolor{my2}{RGB}{0,100,0}
\title{Autoequivariant Network Search via Group Decomposition}

\author{%
  Sourya Basu\\ 
  Department of Electrical\\ and Computer Engineering\\
   University of Illinois at Urbana-Champaign\\
  Illinois, USA \\
  \texttt{sourya@illinois.edu} \\
   \And
   Akshayaa Magesh \thanks{equal contribution} \\ 
   Department of Electrical\\ and Computer Engineering \\
   University of Illinois at Urbana-Champaign \\
   Illinois, USA\\
   \texttt{amagesh2@illinois.edu} \\
   \AND
   Harshit Yadav$^*$\\ 
   Department of Mathematics \\
   Rice University \\
   Texas, USA\\
   \texttt{hy39@rice.edu} \\
   \And
   Lav R.\ Varshney \\
   Department of Electrical\\ and Computer Engineering \\
   University of Illinois at Urbana-Champaign \\
   Illinois, USA\\
   \texttt{varshney@illinois.edu} \\
   }

\begin{document}

\maketitle

\begin{abstract}
Recent works show that 
group equivariance as an inductive bias improves neural network performance for both classification and generation. 
However, designing group-equivariant neural networks is challenging when the group of interest is large and is unknown. Moreover, inducing equivariance can significantly reduce the number of independent parameters in a network with fixed feature size, affecting its overall performance.
We address these problems by proving a new group-theoretic result in the context of equivariant neural networks that shows that a network is equivariant to a large group if and only if it is equivariant to smaller groups from which it is constructed. Using this result, we design a novel fast group equivariant construction algorithm, and a deep Q-learning-based search algorithm in a reduced search space, yielding what we call \emph{autoequivariant networks} (AENs). AENs find the right balance between equivariance and network size when tested on new benchmark datasets, G-MNIST and G-Fashion-MNIST, obtained via group transformations on MNIST and Fashion-MNIST respectively that we release. Extending these results to group convolutional neural networks, where we optimize between equivariances, augmentations, and network sizes, we find group equivariance to be the most dominating factor in all high-performing GCNNs on several datasets like CIFAR10, SVHN, RotMNIST, ASL, EMNIST, and KMNIST.
\end{abstract}

\section{Introduction}\label{sec: introduction}
Recent progress in deep learning owes much of its success to novel network architectures for efficient processing of large datasets. One example for image, video, and audio data is the convolutional neural network (CNN) \cite{LecunBH2015}.
CNNs efficiently process high-dimensional 
data using two key principles: \emph{translation equivariance} and \emph{parameter sharing}. The convolutional layers in CNNs are translation equivariant, so when the input to these layers is shifted in space, the extracted features also get shifted. Translation equivariance encodes the prior that translating an input does not change its labels 
into the neural network 
and helps efficiently extract features from raw data. Parameter sharing then not only efficiently uses parameters, but also helps induce translation equivariance in CNNs. 
\blfootnote{This work was funded in part by the IBM-Illinois Center for Cognitive Computing Systems Research (C3SR), a research collaboration as part of the IBM AI Horizons Network and the National Science Foundation Grant CCF-1717530.}

The idea of equivariance using convolutions and parameter sharing has been generalized to general group symmetries \cite{CohenW2016a,CohenW2016b,RavanbakhshSP2017}. These group-equivariant networks use efficient parameter sharing techniques to encode symmetries in data beyond translation, such as rotations and flips, as priors. Much further research has focused on inducing equivariances for different group symmetries and data types such as \cite{CohenGKW2018} for spherical symmetries, \cite{WorrallW2019, SosnovikSS2019} for scale-equivariance, \cite{Bekkers2019, FinziSIG2020,RomeroBTH2020} for Lie groups, and \cite{HutchinsonLZDTK2020, RomeroH2019, RomeroBTH2020b} for group-equivariance within attention mechanisms.

But all these works assume knowledge of symmetries in the data is known \emph{a priori}. Very recently, \cite{ZhouKF2021} proposed to learn appropriate parameter sharing from the data itself using meta-learning techniques; \cite{DehmamyLWY2021} proposed using $L$-conv layers to automatically search and approximate group convolutions. 
\begin{figure*}
    \centering
    \includegraphics[width=5in]{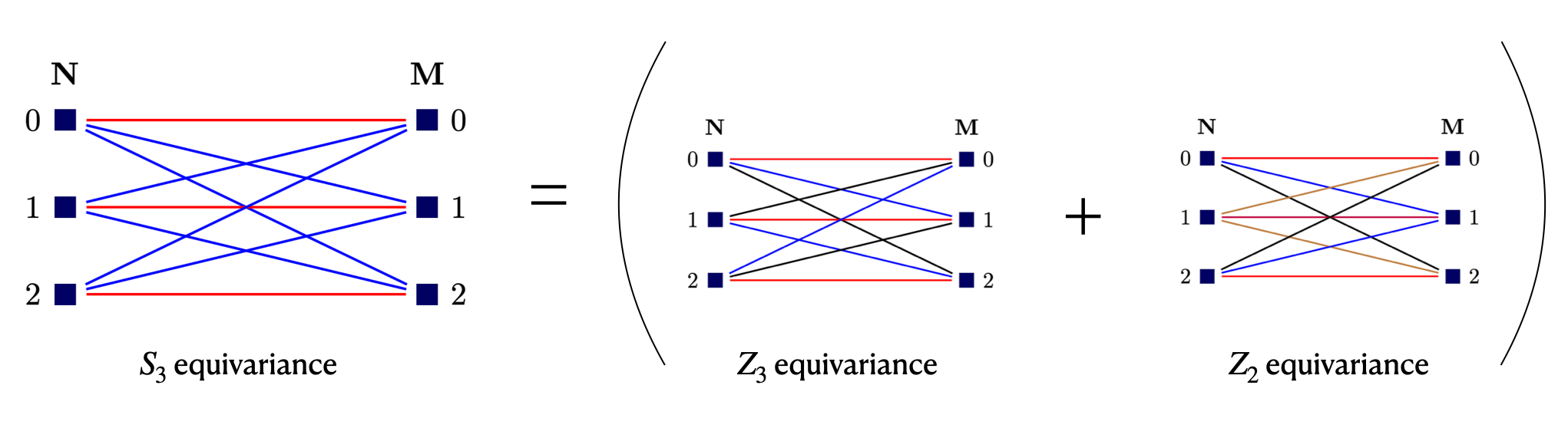}\vspace{-4.5mm}
    \caption{A $3 \times 3$ network with input \textbf{N} and output \textbf{M} is equivariant to a larger group $S_3$ if and only if it is equivariant to the smaller groups $\mathbb{Z}_3$ and $\mathbb{Z}_2$ since $S_3 = \mathbb{Z}_3 \rtimes \mathbb{Z}_2$.}
    \label{fig: introductory_plot}
\end{figure*}
Here, we propose \emph{autoequivariant networks} (AEN) that automatically induce group equivariance from a reduced search space using deep Q-learning, building on new group theory results that we prove. Compared to \cite{ZhouKF2021}, we are able to restrict to a much smaller search space of relevant symmetries by proving an equivalence relation between parameter sharing schemes using large groups and its subgroups that it can be \textit{constructed} from, as illustrated in Fig.~\ref{fig: introductory_plot} (technical details deferred to Sec.~\ref{sec: group_decomposition}). Further, this property proved in Sec.~\ref{sec: group_decomposition} leads to faster construction of equivariant neural networks, as discussed in Sec.~\ref{sec: inducing_equivariance}.
Unlike \cite{DehmamyLWY2021}, we focus on exact symmetries formed by combining several smaller symmetries. 
The overall performance of a network is a function of not only the symmetry in its parameters but also on the number of parameters and features in them. Hence, when the group symmetries are large, then equivariant networks constructed using parameter sharing with fixed number of parameters might have too many features or if we fix the number of features, the number of parameters might be too few. This issue with group equivariant networks was identified by \cite{GensD2014} and reiterated by \cite{CohenW2016a,GordonLBB2020}, which limits their application and makes it difficult to choose the right balance between equivariance, number of parameters, and number of features needed to design high-performing equivariant neural networks with reasonable number of parameters and features. We mitigate this issue by letting our search algorithm find the right balance automatically.

Our contributions are summarized as follows. Sec.~\ref{sec: group_decomposition} proves that a neural network is equivariant with respect to a set of group symmetries if and only if it is equivariant to any group symmetry constructed using semidirect products from them. Using this result,  Sec.~\ref{subsec: fast_equivariant_network_construction} provides an efficient algorithm to induce equivariance for large symmetry groups in multilayered perceptrons (MLPs). Sec.~\ref{subsec: FES} describes our deep Q-learning algorithm for equivariant network architecture search with reduced search space. Sec.~\ref{sec: experiments} develops and releases new group-augmented datasets, G-MNIST and G-Fashion-MNIST, which augment the standard MNIST and Fashion-MNIST datasets using several large symmetry groups. These are used to evaluate AENs and also provide a standard benchmark for further research. Sec.~\ref{sec: experiments} also uses deep Q-learning for searching group equivariant CNNs (GCNNs) of different sizes, group equivariances, and trained using different augmentations on several real datasets like CIFAR10, SVHN, RotMNIST, ASL, EMNIST, and KMNIST. We find that the top-performing GCNNs have several group equivariances in them.

\subsection{Related work}\label{subsec: related_works}
\paragraph{GCNNs and parameter sharing} GCNNs were introduced by \cite{CohenW2016a} and have been extended to various groups for image processing \citep{CohenW2016b,CohenGKW2018,WorrallGTB2017,WeilerC2019,EstevesAMD2018,CohenWKW2019,EstevesMD2020}, where the general idea is to modify the filters as per the groups used and then perform planar convolution. 
An alternative is to construct equivariant MLPs using parameter sharing directly, which divides parameters into relevant orbits to induce equivariance in them \citep{RavanbakhshSP2017}.
Both methods add significant computational load if the size of the group is large.
 Our work provides a method of breaking groups down while preserving equivariance for a wide class of groups and hence making it easier to construct networks equivariant to large symmetries.
GCNNs have also been extended to relevant groups in particle physics \citep{BogatskiyAORMK2020}, reinforcement learning \citep{PolEDHHFW2020}, graphs \citep{HaanWCW2020, HorieNTYN2021,MaronBSL2019b,MaronBSL2019a,KerivenP2019}, and natural language processing \citep{GordonLBB2020}.

\paragraph{Generative and attention models} Group equivariance have also been used in generative models like normalizing flows \citep{FalorsiHDF2019,KohlerKN2020,BilovsG2020}, variational autoencoders \citep{KuzminykhPZ2018,FalorsiHDCWFC2018}, and GANs \citep{DeyCG2021}. These works replace planar convolutions with group convolutions to achieve better representation capabilities.
Work dealing with group convolutions are mainly motivated from the result by \cite{KondorT2018} that shows linear equivariant maps over scalar fields are necessarily convolutional in nature. However, recent work has extended group equivariance to nonlinear equivariant maps such as the popular attention mechanism \citep{VaswaniSPUJGKP2017} for data in the form of point clouds, graphs, images, and language \citep{FuchsWFVW2020,RomeroBTH2020b,HutchinsonLZDTK2020}.
Our main result is very general and works for classification and generation models with linear and nonlinear maps. For 
brevity, this paper only shows experimental and algorithmic results for classification tasks with linear maps, but our results extend to the other forms of equivariant maps as well.

\paragraph{Other forms of equivariant networks} So far we have discussed equivariance in neural networks via various forms of parameter sharing. Let us also point 
out other popular methods of obtaining equivariance 
in the literature. Spatial~\citep{JaderbergSZK2015} and polar transformer networks~\citep{EstevesAZD2018} introduce differentiable modules that are used on top of existing classifiers that transform the input itself to obtain invariance. Capsule networks use dynamic routing instead of pooling to make convolutional networks equivariant \citep{SabourFH2017,LenssenFL2018}, whereas novel pooling schemes 
introduce shift-invariance \cite{Zhang2019, ChamanD2020}.

\paragraph{Group neural architecture search (NAS)} NAS using reinforcement learning yields networks that 
outperform hand-constructed CNNs \citep{BakerGNR2017,ZophL2017,PhamGZLD2018}. Several recent works have proposed automatic architecture construction for group equivariant networks by using meta-learning for reparameterization \citep{ZhouKF2021} or learning generators of groups \citep{DehmamyLWY2021}. A key distinction in our work is in first proving an equivalence relation amongst groups to significantly reduce the search space and then using deep Q-learning to learn relevant symmetries in the data. Further, inducing symmetry in a network varies its number of parameters, which may yield too many / too few parameters, thereby 
reducing 
performance. 
Our Q-learning algorithm instead chooses the appropriate network size 
within a desired range to maximize 
performance.
\paragraph{Efficient group equivariant MLP construction} Recent work  \cite{PolEDHHFW2020} constructs equivariant MLPs by formulating it as an optimization problem with equivariance as its constraints. But the complexity of this algorithm has a dependency on the size of the group. Contemporaneous work in \cite{FinziWW2021} uses a similar optimization framework with novel constraints on generators of groups rather than the group itself to derive efficient construction algorithms that do not depend on the size of the group and outperforms the algorithm in \cite{PolEDHHFW2020}. But \cite{FinziWW2021} does not use the structure of groups like direct or semidirect products and leaves it for further investigation to speed up their algorithm. Our work on the other hand only uses the product structure of groups to speed up the parameter sharing algorithm by \cite{RavanbakhshSP2017}, instead of an optimization framework. We leave investigation into efficient design of equivariant networks using both generators of groups and group structures to future work.
Further, because of the modular structure of our search algorithm in AEN, any such algorithm can be used as a drop-in replacement of our algorithm. Moreover, \cite{FinziWW2021} formulates their problem inducing equivariance in linear networks, whereas our main theorem on group decomposition does not assume linearity.

\section{Preliminaries}\label{sec: preliminaries}
Here we provide definitions related to group equivariance and various products of groups. For more background on groups and group actions please refer to Sec.~\ref{sec: groups} in the supplementary material. 
First define equivariant functions. 
\begin{definition}[Equivariance]\label{def: equivariance}
A function $\phi: \mathcal{X} \mapsto \mathcal{Y}$ is called \emph{$G$-equivariant} (or, \emph{equivariant} to a group $G$) if $\phi(\Gamma_g  x) = \Gamma'_g \phi(x) $ for all $g \in G$ and $x \in \mathcal{X}$, where $\Gamma_g$ and $\Gamma'_g$ are actions of $G$ on $\mathcal{X}$ and $\mathcal{Y}$ respectively. 
\end{definition}
Now define forms of group products: direct products, semidirect products, and central extensions.
\begin{definition}[Direct product]
For two groups $G_1,G_2$, the \emph{direct product} group $G_1 \times G_2$ is defined as the group with underlying set the Cartesian product $G_1 \times G_2$ with (i) {multiplication}: $(g_1,g_2)(g_1',g_2') = (g_1g_1',g_2g_2')$, for $g_1,g_1' \in G_1$ and $g_2,g_2' \in G_2$,
(ii) {identity}: $(e_{G_1},e_{G_2}) \in G_1 \times G_2$ is the identity element,
(iii) {inverse}: $(g_1,g_2)^{-1} = (g_1^{-1},g_2^{-1})$ is the inverse of $(g_1,g_2)$.
\end{definition}
\begin{definition}[Semidirect product] \label{def: semidirect_product}
Let $G_1$ and $G_2$ be groups, and $\alpha: G_2 \rightarrow $Aut$(G_1)$ be a homomorphism of $G_2$ into automorphism group of $G_1$. The \emph{semidirect product} group $G_1 \rtimes_{\alpha} G_2$ is the set $G_1 \times G_2$ under the multiplication $(g_1, g_2)(g_1',g_2') = (g_1\alpha_{g_2}(g_1'),g_2 g_2')$ and identity $(e_{G_1},e_{G_2})$.
\end{definition}
\begin{definition}[Central extension]\label{def: central_extension}
Let $G_1$ be an abelian group and $G_2$ be an arbitrary group. A map $\psi: G_2 \times G_2 \rightarrow G_1$ satisfying
$ \psi(e_{G_2},e_{G_2})=e_{G_1},$
$\psi(g,g'g'') \psi(g',g'') = \psi(g,g')\psi(gg',g'')$
is called a $2$-cocycle. The \emph{central extension} of $G_2$ by $G_1$ (using the $2$-cocycle $\psi$), is defined as the set $G_1\times G_2$ with the group operation $(g_1,g_2)*(g'_1,g'_2) = (g_1 g'_1 \psi(g_2,g'_2),g_2 g'_2),$
where $g_1,g'_1 \in G_1$, $g_2,g'_2 \in G_2$. The identity element is $(e_{G_1},e_{G_2})$.
\end{definition}

\section{Group Decomposition}\label{sec: group_decomposition}
Here we prove our main theoretical results, first discussing them at a high level. We show that a function $\phi$ is equivariant to a group $G$ that can be written as semidirect products of several smaller groups if and only if $\phi$ is equivariant to each of the smaller groups. This implies that for constructing a $G$-equivariant function $\phi$, where $G$ is a possibly large group that can be constructed from several smaller groups using semidirect products, it is sufficient to make the function equivariant to the smaller groups. 
Further, since all the semidirect product groups $G_1 \rtimes_{\alpha} G_2$ for any valid homomorphism $\alpha$ have the same components $G_1,G_2$, we show that a function $\phi$ is equivariant to $G_1 \rtimes_{\alpha} G_2$ if and only if it is equivariant to $G_1 \rtimes_{\alpha_0} G_2$ for any fixed homomorphism $\alpha_0$. Hence, this result gives an equivalence class of groups consisting of  exactly the same decomposition in terms of semidirect products. 
Now we formally state and prove this idea.

\begin{theorem}\label{thm: semidirect_products}
Let $\mu^{G_i}$, $\bar{\mu}^{G_i}$ be group actions of $G_i$ on $\mathcal{X}$ and $\mathcal{Y}$ respectively for $i \in \{1,2\}$. Let $\Gamma$, $\bar{\Gamma}$ be actions of $G_1 \rtimes_{\alpha} G_2$ on $\mathcal{X}$ and $\mathcal{Y}$ respectively for some homomorphism $\alpha: G_2 \mapsto \text{Aut}(G_1)$ such that
\begin{align} 
    \Gamma_{(g_1,e)}(x) &= \mu_{g_1}^{G_1}(x) \hspace{2mm} \forall g_1\in G_1,x\in \mathcal{X}\label{eqn: thm_1_1}\\
    \bar{\Gamma}_{(g_1,e)}(y) &= \bar{\mu}_{g_1}^{G_1}(y) \hspace{2mm} \forall g_1\in G_1,y\in \mathcal{Y} \label{eqn: thm_1_2},\\
    \Gamma_{(e,g_2)}(x) &= \mu_{g_2}^{G_2}(x) \hspace{2mm} \forall g_2\in G_2,x\in \mathcal{X} \label{eqn: thm_1_3}\\
    \bar{\Gamma}_{(e,g_2)}(y) &= \bar{\mu}_{g_2}^{G_2}(y) \hspace{2mm} \forall g_2\in G_2,y\in \mathcal{Y} \label{eqn: thm_1_4},
\end{align}
Then, any function $\phi:\mathcal{X}\mapsto \mathcal{Y}$ is equivariant to $G_1 \rtimes_{\alpha}G_2$ under group actions $(\Gamma,\bar{\Gamma})$ if and only if $\phi$ is equivariant to $G_i$ under group actions $(\mu^{G_i},\bar{\mu}^{G_i})$ for both $i \in \{1,2\}$.
\end{theorem}

In the next sections we describe algorithms using this decomposition. Before that, we show that although the idea of decomposing groups that can be written as semidirect products is quite general and useful for most practical applications, this form of decomposition might fail in other forms of products of groups and needs to be used with caution. To this end, consider the example of central extensions from Def.~\ref{def: central_extension}. We show that it is not trivial to extend the results beyond semidirect products. Recall that in the proof for the semidirect product result in Thm.~\ref{thm: semidirect_products} we used the following equations
$
    \phi(\Gamma_{(g_1,g_2)}(x)) = \phi(\Gamma_{(g_1,e)}\Gamma_{(e,g_2)}(x))
                                =\bar{\Gamma}_{(g_1,e)}\phi(\Gamma_{(e,g_2)}(x)).
$
But for central extensions we have$
    (g_1,e)(e,g_2) = (g_1e\psi(e,g_2),eg_2)
                   = (g_1\psi(e,g_2),g_2),
$
which is not necessarily the same as $(g_1,g_2)$. Hence, here we restrict our discussions to semidirect products, which in itself covers a wide range of groups and is useful for practical purposes. Now we look at a consequence of Thm.~\ref{thm: semidirect_products}.
\begin{corollary}\label{cor: semidirect_products}
Let $\mu^{G_i}$, $\bar{\mu}^{G_i}$ be group actions of $G_i$ on $\mathcal{X}$ and $\mathcal{Y}$ respectively for $i \in \{1,2\}$ as in Thm.~\ref{thm: semidirect_products}. Let $\Delta$, $\bar{\Delta}$ be actions of $G_1 \times G_2$ on $\mathcal{X}$ and $\mathcal{Y}$ respectively such that
\begin{align*}
    \Delta_{(g_1,e)}(x) &= \mu_{g_1}^{G_1}(x), \hspace{2mm} \bar{\Delta}_{(g_1,e)}(y) = \bar{\mu}_{g_1}^{G_1}(y) \hspace{4mm} \forall g_1\in G_1,x\in \mathcal{X},y\in \mathcal{Y} \\
    \Delta_{(e,g_2)}(x) &= \mu_{g_2}^{G_2}(x), \hspace{2mm} \bar{\Delta}_{(e,g_2)}(y) = \bar{\mu}_{g_2}^{G_2}(y) \hspace{4mm} \forall g_2\in G_2,x\in \mathcal{X}, y\in \mathcal{Y}
\end{align*}
Then, any function $\phi:\mathcal{X}\mapsto \mathcal{Y}$ is equivariant to $G_1 \times G_2$ under group actions $(\Delta,\bar{\Delta})$ if and only if $\phi$ is equivariant to $G_i$ under group actions $(\mu^{G_i},\bar{\mu}^{G_i})$ for both $i \in \{1,2\}$.
\end{corollary}
The proof to Cor.~\ref{cor: semidirect_products} follows directly from Thm.~\ref{thm: semidirect_products} since  direct product is a special case of semidirect product. By symmetry of direct product in Cor.~\ref{cor: semidirect_products}, it implies that the order in which equivariance is induced in a network for a set of smaller groups does not matter.
Also note that our decomposition result is purely in the context of neural networks, unlike works in abstract algebra like Jordan-H\"{o}lder \cite{Baumslag2006} and Iwasawa \cite{Iwasawa1949} decompositions that have different algebraic significance.
\section{Inducing Equivariance in MLPs}\label{sec: inducing_equivariance}
We first review a method used by \cite{RavanbakhshSP2017} to induce discrete group equivariances in fully connected layers of MLPs. Then, drawing from Thm.~\ref{thm: semidirect_products}, we show how to extend this method to induce equivariance of a large group by iteratively inducing equivariances corresponding to smaller symmetries that the larger group is constructed from, using semidirect product.

Represent a fully connected bipartite layer of a MLP by a triple $\Omega = (\mathbf{N},\mathbf{M},\Delta)$, where $\mathbf{N}$ and $\mathbf{M}$ are two sets of nodes corresponding to the input and output layers respectively, and $\Delta: \mathbf{N}\times \mathbf{M} \mapsto \mathbf{C}$ is the edge function that assigns color from the set $\mathbf{C}$ to each of the edges. Edges with the same color share the same trainable parameter.
Let the group action of group $G$ on $\mathbf{N}$ and $\mathbf{M}$ be represented by $\Gamma_{\mathbf{N}}$ and $\Gamma_{\mathbf{M}}$ respectively. And let $\Gamma_{\mathbf{N},\mathbf{M}}$ represent the pairing of actions $\Gamma_{\mathbf{N}}$ and $\Gamma_{\mathbf{M}}$, i.e., $\Gamma_{\mathbf{N},\mathbf{M}}$ acts on $(n,m) \in \mathbf{N}\times \mathbf{M}$ to output $(\Gamma_{\mathbf{N}}(n),\Gamma_{\mathbf{M}}(m))$. Hence, $\Gamma_{\mathbf{N},\mathbf{M}}$ can be seen as permuting the edges in $\Omega$ instead of the vertices.
The action of $\Gamma_{\mathbf{N},\mathbf{M}}$ on the set of edges partitions them into orbits $\{\Gamma_{\mathbf{N},\mathbf{M}}(n_p,m_q)\}_{n_p,m_q}$, where $(n_p,m_q)$ is a representative edge of an orbit. Thus, this fully connected layer only has as many independent parameters as the number of orbits that $\Gamma_{\mathbf{N},\mathbf{M}}$ divides $\mathbf{N} \times \mathbf{M}$ into.
Consider the following construction of a fully connected layer using parameter sharing
\begin{align}\label{eqn: single_group_equivariance}
    \Omega = (\mathbf{N},\mathbf{M},\{\Delta(\Gamma_{\mathbf{N},\mathbf{M}}(n_p,m_q))\}),
\end{align}
where $\{\Delta(\Gamma_{\mathbf{N},\mathbf{M}}(n_p,m_q))\}$ is the set of all independent parameters with one independent parameter per orbit. An algorithmic version of this construction is provided in Alg.~2 in the for further discussion later.
Now we define the notion of $G$-equivariance in a fully connected layer $\Omega$.
\begin{definition}\label{def: g-equivariant-neural-network}
Let $\Omega = (\mathbf{N},\mathbf{M},\Delta)$ be a fully connected layer of a MLP taking input $x \in \mathcal{X}$ and producing output  $\Omega(x) = y\in \mathcal{Y}$. Then $\Omega$ is defined to be $G$-equivariant under group actions $\Gamma$ and $\bar{\Gamma}$ of $G$ on $\mathcal{X}$ and $\mathcal{Y}$ respectively if $\Omega(\Gamma_g(x)) = \bar{\Gamma}_g(\Omega(x))$ for all $g \in G$.
\end{definition}
Recall a theorem by \cite{RavanbakhshSP2017} that states that the construction of $\Omega$ given by \eqref{eqn: single_group_equivariance} is $G$-equivariant.
\begin{theorem}[\cite{RavanbakhshSP2017}]\label{thm: ravanbaksh}
The neural network construction $\Omega$ given by \eqref{eqn: single_group_equivariance} is equivariant to $G$.
\end{theorem}
Now consider the construction of $\Omega$ in \eqref{eqn: single_group_equivariance} when the group $G$ is of the form $G = G_1 \rtimes_{\alpha} G_2$. We claim that in this case, a $G$-equivariant fully connected layer $\Omega$ can be obtained by having $\Omega$ respect the parameter sharing laws in \eqref{eqn: single_group_equivariance} for $G_1$ and $G_2$ both. This claim follows directly from Thm.~\ref{thm: semidirect_products}. This result can further be extended to any group $G$ constructed by using semidirect product on groups from a set $\mathcal{G} = \{G_1,G_2,\ldots,G_m\}$. We leverage this result to get an efficient algorithm for construction of equivariant networks for large groups in Sec.~\ref{subsec: fast_equivariant_network_construction}.
We will further use the symmetry of arguments $G_1,G_2$ in Cor.~\ref{cor: semidirect_products} for searching large symmetries in Sec.~\ref{subsec: FES}. In particular, we will use the fact that the order in which the group equivariances for groups from the set $\mathcal{G}$ are induced does not matter, which reduces the search space for our algorithm.

\paragraph{Example}
Before looking at various construction and search algorithms, 
consider some examples demonstrating an implication of Thm.~\ref{thm: semidirect_products}.
Consider an example with $\mathbf{N} = \mathbf{M} = \{1,2,3\}$, i.e.\ a $3\times 3$ fully connected layer $\Omega$. First we construct a symmetric group $S_3$-equivariant network by using the design in \eqref{eqn: single_group_equivariance}.
In Fig.~\ref{fig: introductory_plot}, for $S_3$-equivariance, there are only two independent parameters illustrated using red and blue colors.
Now, consider an alternate construction of the same network by first inducing $\mathbb{Z}_3$ cyclic equivariance followed by parameter sharing rules for $\mathbb{Z}_2$ flip equivariance as shown in Fig.~\ref{fig: introductory_plot}. Note that by combining the parameter sharing restrictions for both $\mathbb{Z}_3$-equivariance and $\mathbb{Z}_2$-equivariance, we obtain the parameter sharing  obtained for $S_3$-equivariance. This can be explained from the fact $S_3 = \mathbb{Z}_3 \rtimes_{\alpha} \mathbb{Z}_2$ for some $\alpha$ and a consequence of Thm.~\ref{thm: semidirect_products}. Hence, breaking symmetries down helps reduce the search space of groups by introducing equivalence relations between them
\subsection{Fast Equivariant Network Construction}\label{subsec: fast_equivariant_network_construction}
Here we give a construction algorithm, Alg.~\ref{alg: fast-equivariant-network-construction}, for inducing equivariance in fully connected layers. For a group $G$ constructed by iteratively using semidirect product of $m$ smaller groups in any order from an array $G_{array} = [G_1,G_2,\ldots,G_m]$, inducing $G$-equivariance through the parameter sharing method in \eqref{eqn: single_group_equivariance} has $\mathcal{O}(|G|\times N)$ computational complexity, where, clearly $|G| = |G_1| \times \dots \times |G_m|$. We use Thm.~\ref{thm: semidirect_products} to reduce this complexity to $\mathcal{O}\left( (|G_1|+\cdots+|G_m|)\times N \right)$. Note that trivially using \eqref{eqn: single_group_equivariance} for each $G_i \in G_{array}$ to induce equivariance might not give a computational advantage. For example, for $G_{array} = [G_1,G_2]$, individually inducing equivariance would give us two sets of orbits of size $\mathcal{O}(N/|G_i|)$. But, then merging these two sets of orbits with arbitrary partitioning of indices in each orbit would take $\mathcal{O}(N^2/(|G_1|\times |G_2|))$ computations, which for most experimental cases considered in this paper is impractical. Now we describe Alg.~\ref{alg: fast-equivariant-network-construction} that produces a $G$-equivariant network with $\mathcal{O}\left( (|G_1|+\cdots+|G_m|)\times N \right)$ computational complexity. The main idea behind Alg.~\ref{alg: fast-equivariant-network-construction} is to take inspiration from Thm.~\ref{thm: semidirect_products} and produce the orbits that resemble a $G$-equivariant network, but visiting each element of $G_i$ for $i \in \{1,\ldots,m\}$ exactly once per number of parameters.

\paragraph{Notation for Algs.~\ref{alg: fast-equivariant-network-construction} and 2}
The total number of parameters in the network is $N$. Define an array $\mathbf{V}$ of length $N$ which is initialized to all $-1$s. When the $i$th parameter is visited, $\mathbf{V}[i]$ is set to $1$. The sharing amongst  parameters is indicated by an array, $\mathbf{I}$, of length $N$, which is initialized to $\mathbf{I}[i] = i$, and $\mathbf{I}[i] = j$ indicates  the $i$th parameter belongs to the $j$th orbit. An integer $C$ initialized to $-1$ is the current orbit number. The array of small groups that construct $G$ is stored in an array $G_{array}$ of size $g_{size}$. Define another array $\mathbf{L}$ with $\mathbf{L}[i]=i$ and an empty queue, $Q$, to enumerate parameters.
\begin{algorithm}[t]
\begin{algorithmic}
\FOR{$i \in \mathbf{L}$}
    \IF{$\mathbf{V}[i]<0$}
    \STATE $\mathbf{V}[i] = 1$
    \STATE $Q.append(i)$
    \STATE $C = C+1$
    \WHILE{$len(Q)>0$}
    \STATE $index = Q.pop()$
    \STATE $I[index]=C$
    \FOR{$G \in G_{array}$}
    \FOR{$g \in G$}
    \STATE $index_g = \Gamma_g(index)$
    \IF{$V[index_g]<0$}
    \STATE $V[index_g]=1$
    \STATE $Q.append(index_g)$
    \STATE $I[index_g] = C$
    \ENDIF
    \ENDFOR
    \ENDFOR
    \ENDWHILE
    \ENDIF
    \ENDFOR
\Return $\mathbf{I},C$
\end{algorithmic}
 \caption{Fast Equivariant MLP Construction}
 \label{alg: fast-equivariant-network-construction}
\end{algorithm}

In Alg.~2 in Sec.~\ref{sec: basic_algorithm} in the supplementary material, we give an algorithmic version of the construction of a G-equivariant MLP in \eqref{eqn: single_group_equivariance} using the same notations as in Alg.~\ref{alg: fast-equivariant-network-construction} to illustrate the difference between the two algorithms in Fig.~\ref{fig: Fast_construction_illustration}. Take the group $G$ to be of the form $G=G_1\rtimes G_2$, then the main difference between the two algorithms is that in Alg.~2 for each index $i\in L$, we iterate through all elements of $G$ to construct the orbits, whereas in Alg.~\ref{alg: fast-equivariant-network-construction} for each index $i \in L$, we only go through the elements of $G_1$ and $G_2$. Fig.~\ref{fig: Fast_construction_illustration} illustrates this difference

Now we prove correctness and computational complexity of Alg.~\ref{alg: fast-equivariant-network-construction}.  Claim~\ref{claim: correctness}  shows that for a weight matrix $\mathbf{W}$ with $N$ parameters, the matrix $\mathbf{W}[\mathbf{I}]$ is $G$-equivariant, where $\mathbf{W}[\mathbf{I}]$ is a matrix satisfying the parameter sharing introduced by $I$ and has $C$ total trainable parameters, the total number of orbits returned by Alg.~\ref{alg: fast-equivariant-network-construction}. The proof to which is provided in Sec.~\ref{sec: proofs} in the supplementary material. Claim~\ref{claim: computational complexity} shows that Alg.~\ref{alg: fast-equivariant-network-construction} has computational complexity $\mathcal{O}\left( (|G_1|+\cdots+|G_m|)\times N \right)$.
\begin{figure*}
    \centering
    \includegraphics[width=10cm]{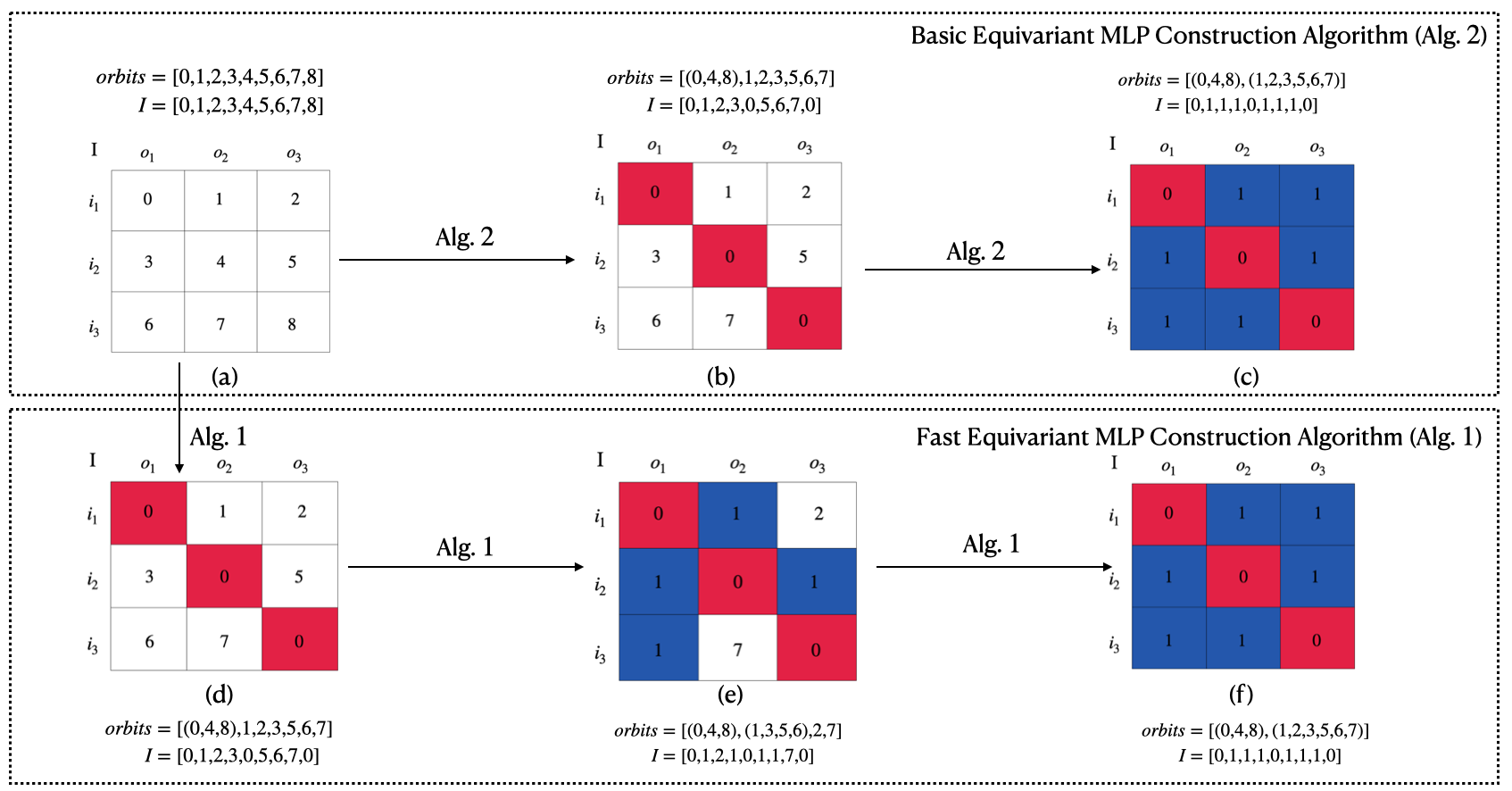}
    \caption{Steps in inducing equivariance of $\mathbb{Z}_3 \times \mathbb{Z}_2$ in a $3\times 3$ MLP using Alg.~\ref{alg: fast-equivariant-network-construction} as compared to a basic approach in Alg.~2. Here (a) represents the orbits of a singles-layered $3 \times 3$ MLP as the starting point for both the algorithms. Alg.~\ref{alg: fast-equivariant-network-construction} takes the steps (d), (e), and (f) in that order, whereas Alg.~2 takes step (b) followed by (c) to obtain the same equivariant MLP. Both methods give the same equivariant network as proved in Claim.~\ref{claim: correctness}, but Alg.~\ref{alg: fast-equivariant-network-construction} is significantly faster than Alg.~2 as proved in Claim.~\ref{claim: computational complexity}}
\label{fig: Fast_construction_illustration}
\end{figure*}
\begin{claim}\label{claim: correctness}
 The weight matrix with parameter sharing given by $\mathbf{W}[\mathbf{I}]$ is $G$-equivariant. Here, $\mathbf{W}[\mathbf{I}]$ is a weight matrix with $N$ parameters sharing $C$ trainable parameters using $\mathbf{I}$ returned by Alg.~\ref{alg: fast-equivariant-network-construction}.
\end{claim}
\begin{claim}\label{claim: computational complexity}
The computational complexity of Alg.~\ref{alg: fast-equivariant-network-construction} is $\mathcal{O}\left( (|G_1|+\cdots+|G_m|)\times N \right)$.
\end{claim}
\begin{proof}
Each index $i$ gets appended to $Q$ exactly once and when $i$ is popped out, using group transformations $\Gamma_g$s it checks at most $|G_1|+\cdots+|G_m|$ more indices to add to Q.
\end{proof}
\subsection{Fast Equivariance Search}\label{subsec: FES}
We use deep Q-learning \cite{SuttonB2018} with $\epsilon$-greedy strategy \cite{VermorelM2005} and experience replay \cite{Lin1993, AdamBB2011} as our search algorithm. In the $\epsilon$-greedy strategy, the agent starts searching for networks with a large value of $\epsilon$ close to 1 when it \emph{explores} various networks and slowly decreases $\epsilon$ to \emph{exploit} better performing networks as the search progresses. Experience replay helps store large numbers of transitions in the state-action space and reuse them in a decorrelated way, which empirically performs well \cite{AdamBB2011}.
We represent a network topology by a state vector describing equivariant properties of the network, which acts as the state of the Q-learning agent. At each step of learning, the agent moves from current state to another by changing group equivariances present in the network.
Various components of our $Q$-learning agent including  search space, action space, evaluation strategy, reward function, and training strategy follow  \cite{BakerGNR2017} and are detailed in Sec.~\ref{sec: fes_details} of the supplementary materials.
\section{Experiments}\label{sec: experiments}
We conduct two sets of experiments in Sec.~\ref{subsec: exp_Group_FC_Nets} and \ref{subsec: exp_Conv_Nets} for group equivariant MLPs and CNNs respectively.
The dataset construction for Sec.~\ref{subsec: exp_Group_FC_Nets} is described in Sec.~\ref{subsubsec: fc_datasets} in the supplementary material using various group transformations on MNIST and Fashion-MNIST datasets.
We report the performance of group equivariant MLPs with fixed number of features and varying number of parameters, where we use one or more groups from the same set of group transformations for inducing equivariance as used for dataset construction. Group equivariances are induced in these networks using deep Q-learning without knowledge of transformations present in the datasets, but only the possible set of groups that can be used to construct it. For a set of symmetries, a large number of groups can be constructed from their semidirect products by changing the automorphism and components involved in the product.
These experiments show that in general, using the same group for inducing equivariance actually in the dataset tends to perform well. But, cases where number of group symmetries are too many, our search algorithm does not choose MLPs with all the relevant symmetries, rather finds a balance between symmetry and parameters.
In Sec.~\ref{subsec: exp_Conv_Nets}, our deep Q-learning algorithm chooses appropriate equivariances, augmentations, and channel sizes in GCNNs to maximize their performance on various image datasets. Code is available at \url{https://github.com/basusourya/autoequivariant_networks}.

\subsection{Group Equivariant MLPs}\label{subsec: exp_Group_FC_Nets}
We first perform single equivariance tests where we use the G-MNIST and G-Fashion-MNIST datasets with single group transformations from Tab.~\ref{tab: individual_augmentations}, followed by deep Q-learning, where we consider the datasets with combination of group transformations from Tab.~\ref{tab: augmentations}.

\paragraph{Single Equivariance Testing}
Here we perform single equivariance tests on G-MNIST and G-Fashion-MNIST datasets as shown in Fig.~\ref{fig: I_G_MNIST} and Fig.~\ref{fig: I_G_Fasion_MNIST} in Sec.~\ref{subsubsec: fc_datasets} in the supplementary material respectively.
The training details are given in Sec.~\ref{sec: training_details_gemlp} in the supplementary material.
From the accuracies in Tab.~\ref{tab: hypothesis_testing_IAug} and Tab.~\ref{tab: fashionMNIST_hypothesis_testing_IAug} we find the following key insights.
We find that test accuracies are a function of equivariances and number of parameters in them. Appropriate equivariances in MLPs 
perform well, but, when the number of parameters are too few for single equivariance, as in the case of Eq4 and Eq5 in Tab.~\ref{tab: parameters_IEq}, the networks  perform poorly across all the datasets. We conclude that equivariance helps in general, but a balance between number of parameters and symmetry induced is important. Hence, in the next subsection, we use deep Q-learning to automatically find the appropriate equivariances with balanced number of parameters to induce in an MLP.

\begin{table*}
\caption{Trainable parameters in MLPs of dimension $784\times 400 \times 400 \times 10$ with various equivariances.}
    \centering
    
    \begin{tabular}{|p{2.3cm} p{0.45cm} p{0.45cm} p{0.45cm} p{0.45cm} p{0.45cm} p{0.45cm} p{0.45cm} p{0.45cm} p{0.45cm} p{0.45cm} p{0.5cm} p{0.5cm} p{0.52cm}|}
    \hline
              \textbf{}  & Eq0 & Eq1 & Eq2 & Eq3 & Eq4 & Eq5 & Eq6 & Eq7 & Eq8 & Eq9 & Eq10 & Eq11 & Eq12 \\
              \hline
        \textbf{Params ($\times 10^6$)} & 0.47& 0.12& 0.24 & 0.24 & 0.04 & 0.04 & 0.24 & 0.30 & 0.30 & 0.30 & 0.30 & 0.12 & 0.24 \\
        \hline
    \end{tabular}
    \label{tab: parameters_IEq}
\end{table*}

\paragraph{Deep Q-learning}
Tab.~\ref{tab: hypothesis_testing_Aug_summary} gives the results for group neural architecture search using deep Q-learning for G-MNIST and G-Fashion-MNIST for augmentations Aug3 and Aug5. More detailed results are presented in Tab.~\ref{tab: MNIST_hypothesis_testing_Aug} and Tab.~\ref{tab: FashionMNIST_hypothesis_testing_Aug}, and training details are provided in Sec.~\ref{sec: training_details_gemlp} in supplementary material.
Each network in the columns are trained for datasets of size 10K for 10 epochs. Now we describe the method to obtain the AEN results.
We train our DQN with the hyperparameters described in Sec.~\ref{sec: fes_details}.  Our DQN 
takes in the current child network state and outputs an action, using which we get the next state of the child network. The reward is a function of the performance of the child network as described in Sec.~\ref{sec: fes_details}.
To make our search efficient, following \cite{BakerGNR2017}, we store the rewards from each state and reuse them when a state is revisited. Fig.~\ref{fig: DQN_summary} illustrate the average test accuracies obtained by the child model per value of $\epsilon$ for Aug3 and Aug5 augmented G-MNIST and G-Fashion-MNIST datasets. 
The complete set of plots are illustrated in Fig.~\ref{fig: DQN_MNIST} and Fig.~\ref{fig: DQN_FashionMNIST} for the two datasets. 
Each DQN was trained for 12 hours on an Nvidia V100 GPU available on HAL cluster~\cite{KindratenkoMZMHRXCPG2020}.
Once the DQNs are trained, we use the top states obtained and retrain them using identical training hyperparameters as the MLPs in the other columns to obtain the results in Tab.~\ref{tab: hypothesis_testing_Aug_summary}, Tab.~\ref{tab: MNIST_hypothesis_testing_Aug} and Tab.~\ref{tab: FashionMNIST_hypothesis_testing_Aug} under the column AEN. Further training details and top states obtained are given in Sec.~\ref{sec: training_details_gemlp} of the supplementary material.
In Tab.~\ref{tab: MNIST_hypothesis_testing_Aug} and Tab.~\ref{tab: FashionMNIST_hypothesis_testing_Aug}, we find that test accuracies obtained from DQN outperform any other network states in the tables from single equivariances in most of the cases. 
\begin{table}
\caption{Test accuracies for equivariant MLPs for G-MNIST and G-Fashion-MNIST with augmentations Aug$i$ corresponding to the array of transformations from Tab.~\ref{tab: augmentations} and the columns indicate equivariant MLPs with Eq$i$ having equivariance to transformation IAug$i$ from Tab.~\ref{tab: individual_augmentations}.}
    \centering
    \begin{tabular}{p{0.6cm} p{0.5cm} p{0.5cm} p{0.5cm} p{0.5cm} p{0.5cm} p{0.5cm} p{0.5cm} p{0.5cm} p{0.5cm} p{0.5cm} p{0.5cm} p{0.5cm} p{0.5cm} p{0.5cm}}
              \textbf{}  & Eq0 & Eq1 & Eq2 & Eq3 & Eq4 & Eq5 & Eq6 & Eq7 & Eq8 & Eq9 & Eq10 & Eq11 & Eq12 & AEN \\
              \hline
        \underline{MNIST}  &    &    &   &    &     &    &    &    &    &    &    &    &    &   \\
        \textbf{Aug3}&73.7& 73.9& 73.9& 72.1& 46.3& 41.4& 72.7& 69.4& 70.1& 75.3& 68.1& 71.8& 73.2& \textbf{77.4}\\
        \textbf{Aug5}&51.0& 57.6& 54.4& 55.8& 38.6& 37.8& 54.8& 54.3& 53.9& 54.8& 55.2& 62.2& 56.2& \textbf{69.4}\\
        \underline{Fashion}& \underline{  -MNIST}  &    &    &   &    &     &    &    &    &    &    &    &    &\\
        \textbf{Aug3}  &58.0& 59.8& 58.2& 57.8& 43.5& 41.2& 58.0& 59.8& 58.8& 60.0& 57.7& 59.7& 57.1 &\textbf{63.4}\\
        \textbf{Aug5}  &50.8& 54.6& 53.1& 52.5& 36.6& 34.6& 53.8& 54.1& 53.3& 53.5& 52.6& 57.8& 52.3 & \textbf{60.4}\\
    \end{tabular}
    \label{tab: hypothesis_testing_Aug_summary}
\end{table}
\subsection{Group Equivariant Convolutional Neural Networks}\label{subsec: exp_Conv_Nets}
We use deep Q-learning on GCNNs with varying group equivariances, augmentations, and size. We replace the child network in Sec.~\ref{subsec: exp_Group_FC_Nets} with a GCNN with symmetries from rotation (P4), horizontal flips (H2), and vertical flips (V2). 
We also allow two network sizes: small and large with the large network having 1.5 times the number of channels as the small network. This is to understand the impact of number of parameters versus other factors like equivariance and augmentations. We adjust the number of channels in GCNNs so for any combination of groups from \{P4, H2, V2\} the number of parameters are nearly equal for the small network. The DQN also chooses augmentations from Tab.~\ref{tab: equivariances_augmentations_GCNN} for training the child network. Thus, we directly compare the effect of augmentations and equivariances in the overall performance. We train a fixed number of models per value of $\epsilon$ as shown in Tab.~\ref{tab: epsilon-models-trained}. Fig.~\ref{fig: DQN_summary} shows the test accuracies obtained during the deep Q-learning process averaged over $\epsilon$ values for CIFAR10~\cite{KrizhevskyH2009} and EMNIST~\cite{CohenATV2017}. A complete set of plots for more datasets SVHN~\cite{NetzerWCBWN2011}, RotMNIST~\cite{LarochelleECABB2007}, ASL~\cite{tecperson}, KMNIST~\cite{ClanuwatBKLYH2018} is given in Sec.~\ref{sec: additional_plots} of the supplementary material.
\begin{figure}%
    \centering
    \subfloat[Aug3 MNIST]{{\includegraphics[width=4.2cm]{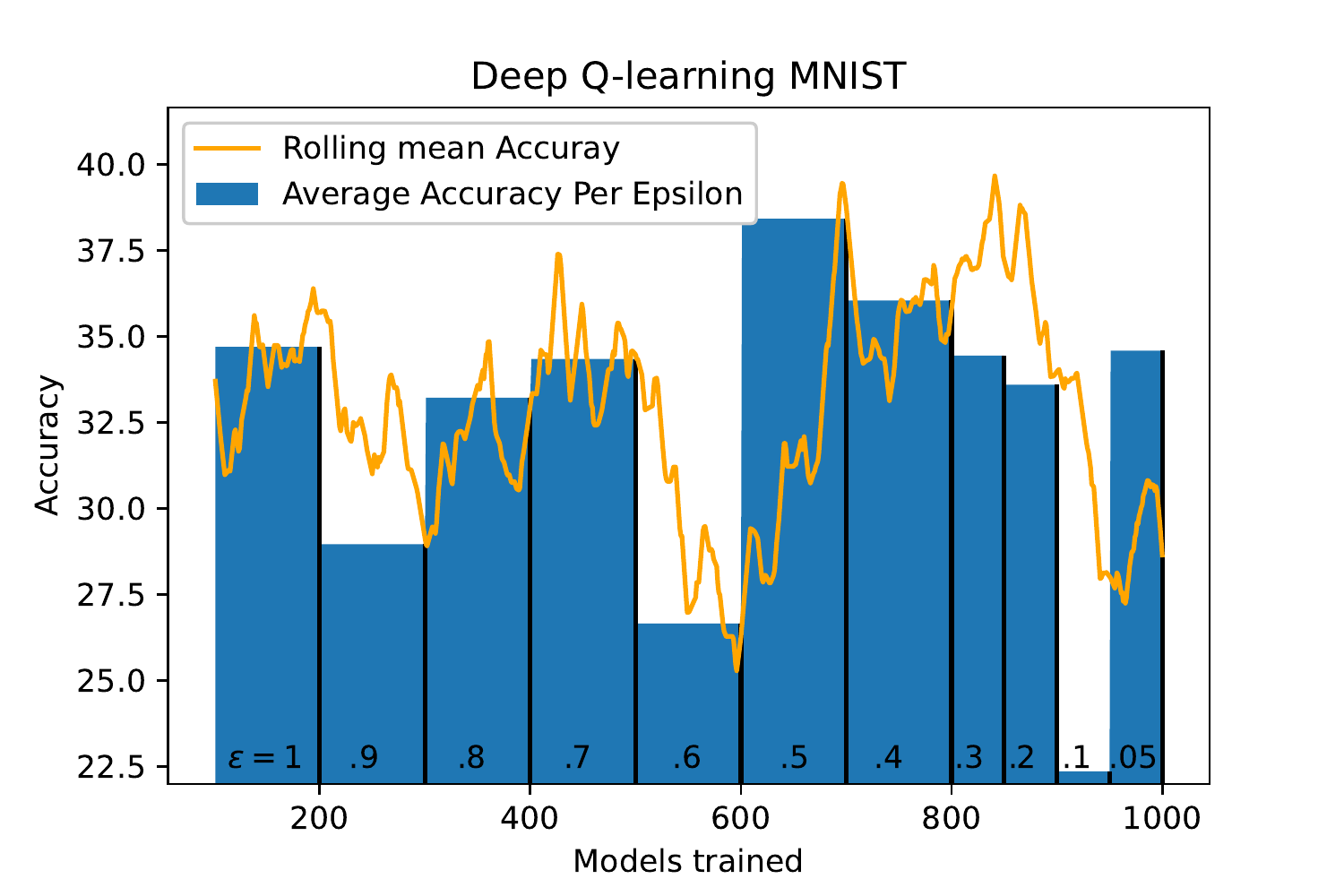} }\label{fig: DQN_MNIST_Aug3}}
    \subfloat[Aug3 Fashion-MNIST]{{\includegraphics[width=4.2cm]{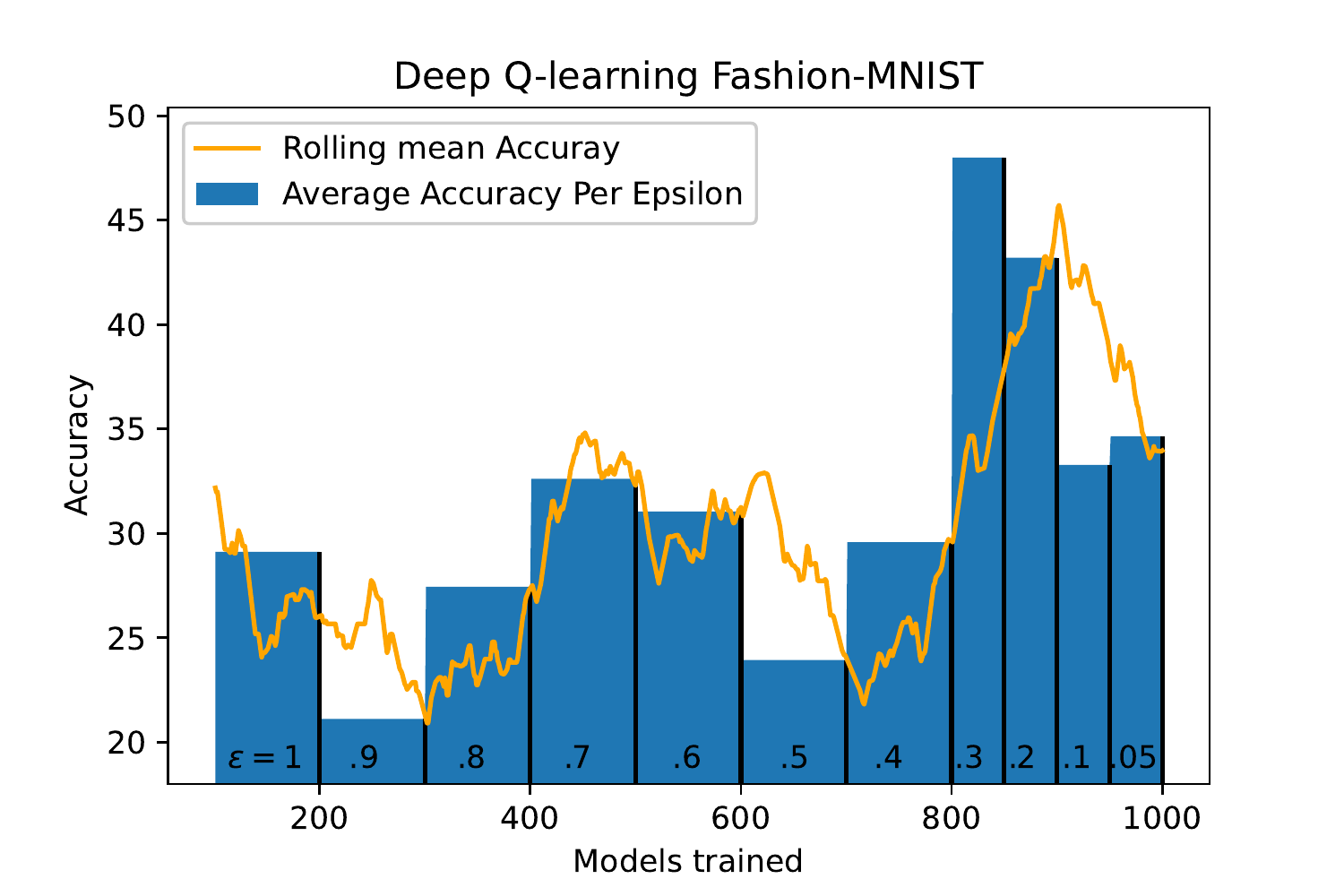} }\label{fig: DQN_FashionMNIST_Aug3}}
    \subfloat[Aug5 MNIST]{{\includegraphics[width=4.2cm]{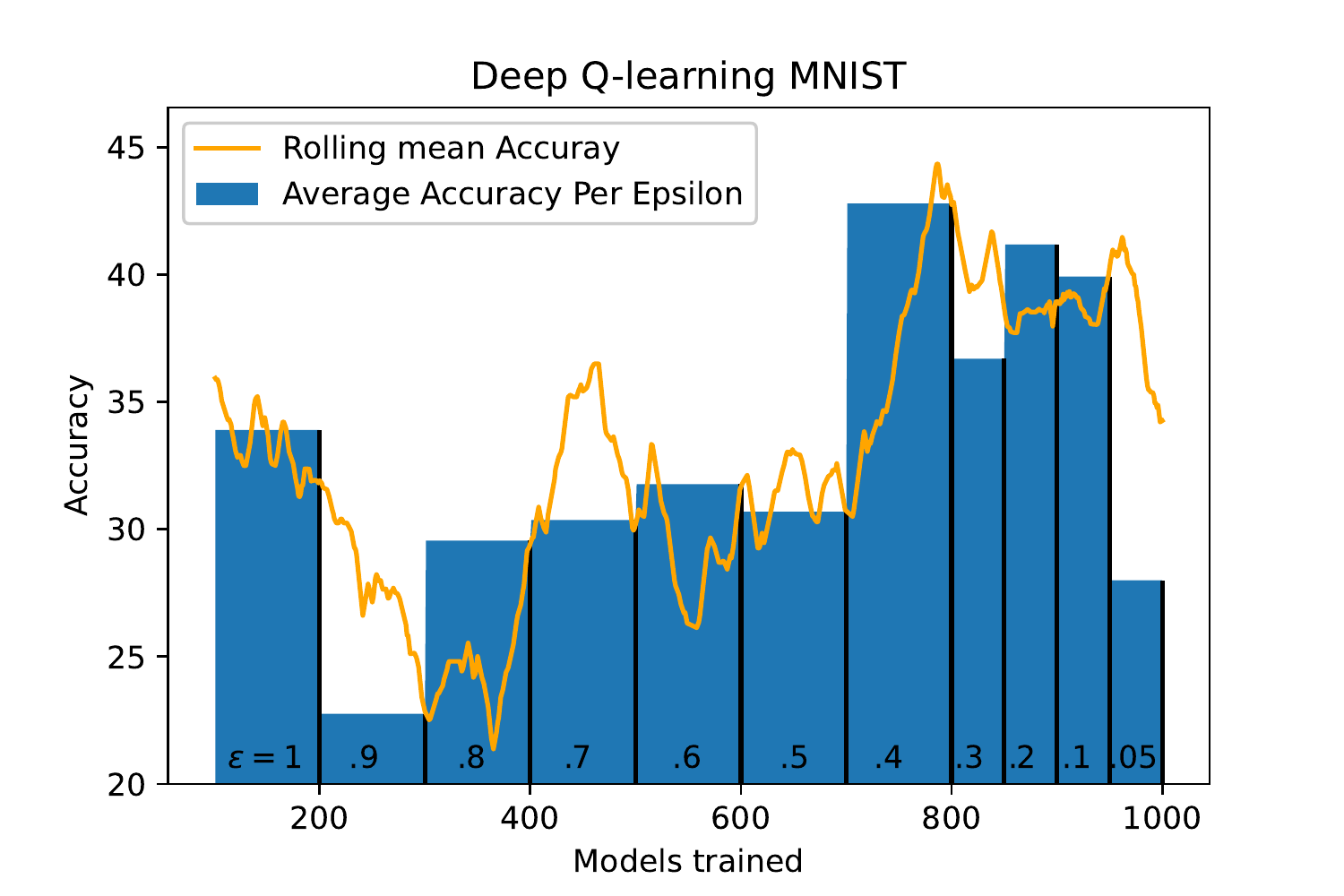} }\label{fig: DQN_MNIST_Aug5}}\\
    \subfloat[Aug5 Fashion-MNIST]{{\includegraphics[width=4.2cm]{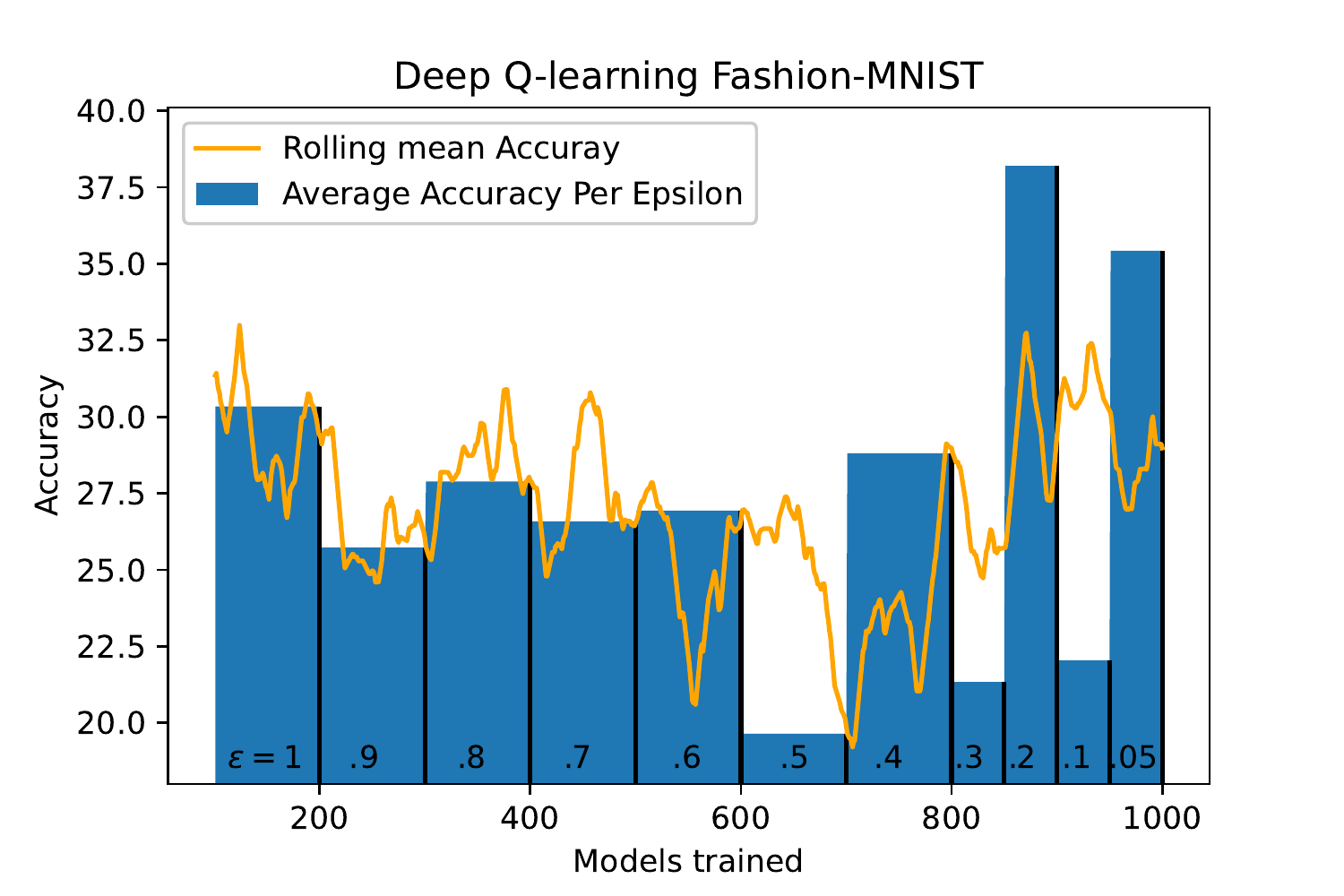} }\label{fig: DQN_FashionMNIST_Aug5}}
    \subfloat[CIFAR10]{{\includegraphics[width=4.2cm]{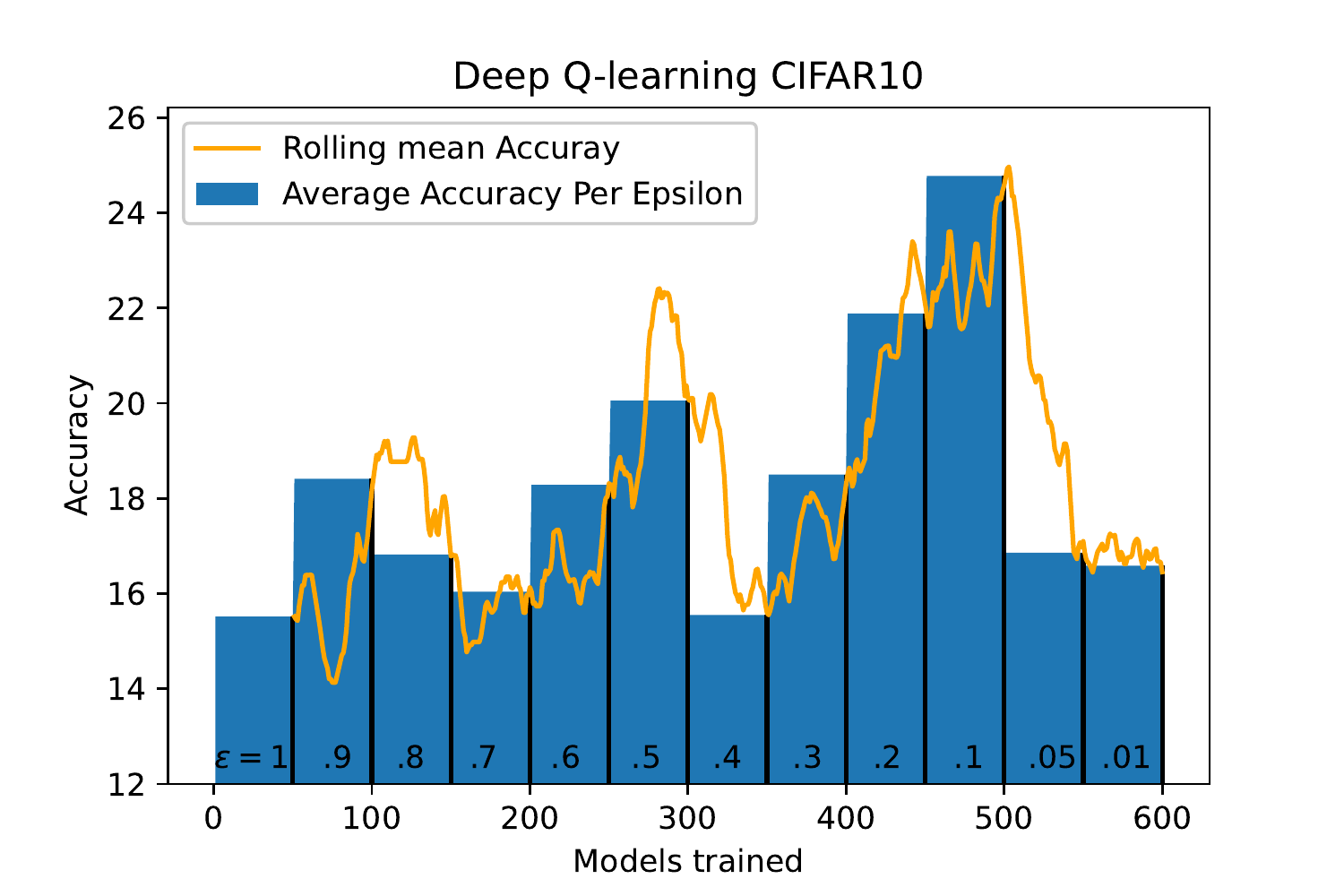} }\label{fig: DQN_CIFAR10}}
    \subfloat[EMNIST]{{\includegraphics[width=4.2cm]{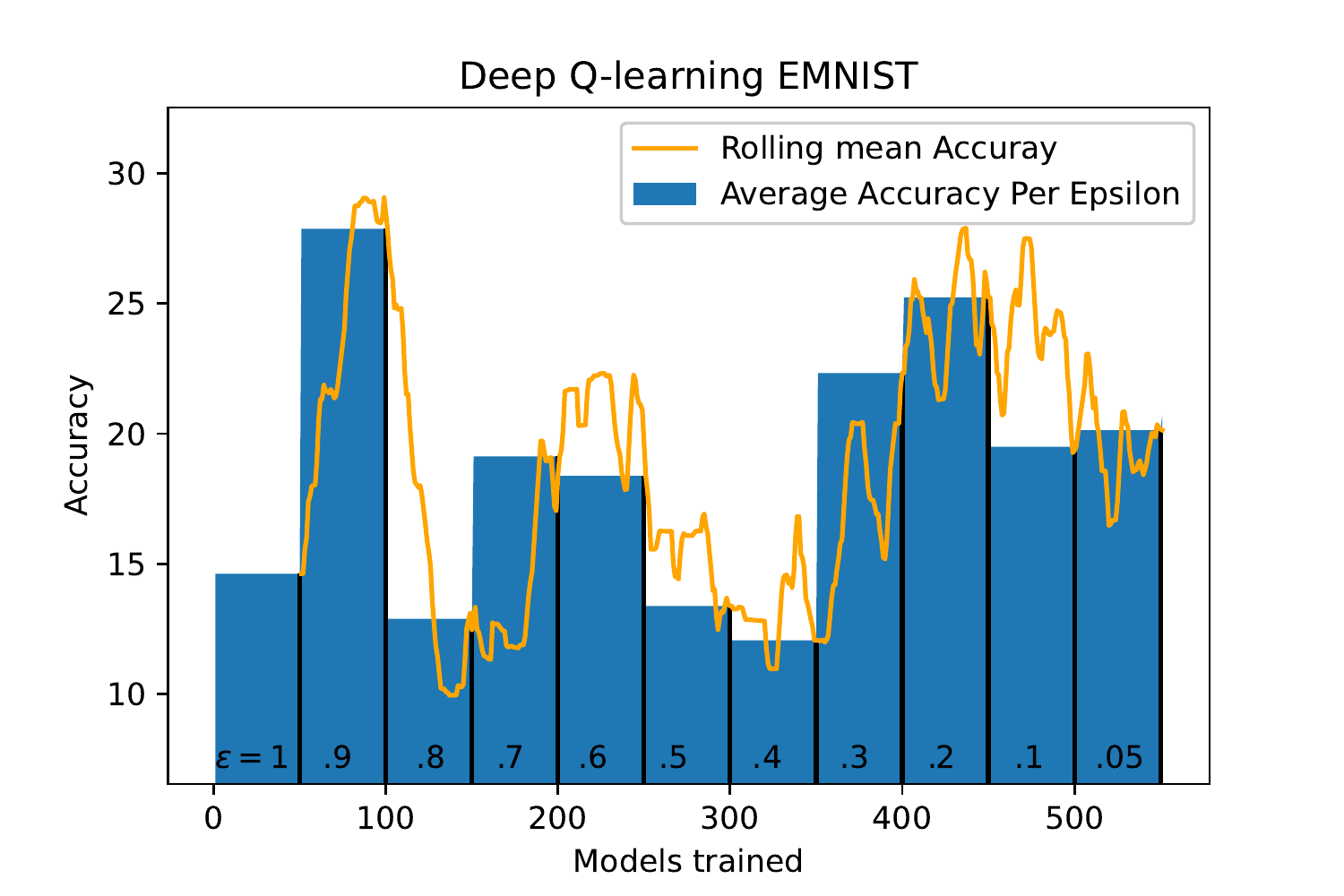} }\label{fig: DQN_EMNIST}}
    \caption{Deep Q-learning performance for various image datasets. The Q-learning agent chooses from various group equivariances for G-MNIST and G-Fashion-MNIST, and from various equivariances, augmentations, and network sizes for CIFAR10 and EMNIST.}%
    \label{fig: DQN_summary}
\end{figure}
\begin{table}
\caption{Top states from deep Q-learning for equivariances, augmentations, and sizes used in training a GCNN on several image datasets and comparison with translation (Zn) equivariant CNNs.}
    \centering
    \begin{tabular}{|p{1.8cm} p{2.4cm} p{3.5cm} p{0.8cm} p{1.4cm} p{1.4cm}|}
    \hline

        \textbf{Dataset} & Equivariances & Augmentations & Size & Test & Base (Zn) \\
                         &               &               &      & accuracies & accuracies \\
        \hline
         {CIFAR10} & (P4, H2, V2, Zn) & (HFlip (0.5), Scale (0.1))  & Large & $65.25\%$ & $55.17\%$\\
        \hline
        {SVHN}    & (P4, H2, V2, Zn) & --- & Large & $89.70\%$ & $85.72\%$ \\
        \hline
        {RotMNIST}& (P4, H2, V2, Zn) & (Scale (0.1)) & Small & $94.57\%$ & $92.60\%$ \\
        \hline
        {ASL}     & (P4, H2, V2, Zn) & (Scale (0.1)) & Small & $94.21\%$ & $89.33\%$\\
        \hline
        {EMNIST}  & (P4, V2, Zn) & (Shear (10)) & Large & $93.99\%$ & $92.93\%$ \\
        \hline
        {KMNIST}  & (P4, H2, V2, Zn) & (Shear (10)) & Large & $93.59\%$ & $90.94\%$ \\
        \hline
    \end{tabular}
    \label{tab: gnas_gcnn_results_summary}
\end{table}
We obtain the top 4 states of the child network with highest rewards from the DQN and retrain them on entire datasets for 10 epochs and report the top states and their performances in Tab.~\ref{tab: gnas_gcnn_results} in the supplementary material. Tab.~\ref{tab: gnas_gcnn_results_summary} summarizes the results for the best pest performing state amongst the top 4 states for each dataset and compares it with a large size traditional CNN. From Tab.~\ref{tab: gnas_gcnn_results} we observe that all  top performing models have group equivariances in them, showing how critical it is for performance in each of the datasets considered. Also note that augmentations do appear in the top models, but not as frequently: although augmentations help with invariance, perhaps they also consume network complexity to learn it. Moreover, we find there is not much impact on the size of the network for a significant change in the network size (1.5 times) since the network size seems to be almost uniformly distributed across the rows in Tab.~\ref{tab: gnas_gcnn_results}. As a takeaway, we conclude that equivariance can induce symmetry without using any network complexity and hence gives superior performance to augmentations.
\section{Conclusion, limitations, and societal impacts}\label{sec: limitations_and_societal_impacts}
We prove new group-theoretic results for improving the complexity of construction and search of group-equivariant neural networks. We find experimentally that AENs, in general, outperform other networks across all the datasets we tested on: G-MNIST, G-Fashion-MNIST, CIFAR10, SVHN, RotMNIST, ASL, EMNIST, and KMNIST.
We work with discrete symmetries since we use parameter sharing, which is developed only for discrete symmetries. In future work, we wish to generalize our methods and extend our work to continuous symmetries. We also assume that we know the set of groups from which the group symmetries are constructed using semi-direct products. In future work, we want to generalize our work to other forms of group extension methods in group theory.
Our work improves the complexity of construction and search of equivariant networks.
Our work does not have direct negative societal impact we are aware of, beyond the general AI safety and governance issues identified in the literature \cite{VarshneyKS2019} from improved AI models. 
NAS does aim to optimize neural networks as usually done by human researchers, but in the current scale of our experiments we do not think it poses a threat to their employment. 


\bibliography{FastEquivarianceSearch.bib}
\bibliographystyle{IEEEtran}
\section*{Checklist}

\appendix

\section{Background on Groups}\label{sec: groups}
Here we provide basic definitions on groups and group actions.
\begin{definition}[Group] \label{def: group}
A group $(G,\circ)$ is a set $G$ equipped with a binary operator $\circ$ that satisfies the following axioms.
\begin{itemize}
\item{Closure:} For any $g_1,g_2 \in G$, $g_1\circ g_2 \in G$. 
\item{Associativity:} For any $g_1,g_2,g_3 \in G$, $g_1\circ (g_2\circ g_3) = (g_1\circ g_2)\circ g_3$. 
\item{Identity}: There exists an $e \in G$, such that $g \circ e = e \circ g = g$ for all $g \in G$.
\item{Inverse:} For any $g \in G$, there is a corresponding $g^{-1}$, such that $g \circ g^{-1} = g^{-1} \circ g = e$.
\end{itemize}
\end{definition}
When clear from context, we omit the symbol $\circ$ and write $g_1 \circ g_2$ as $g_1g_2$. A group is called an abelian group if the group operation commutes.
Next we define a homomorphism between two groups that is useful to define group actions.
\begin{definition}[Sym$(\mathcal{X})$]
For any set $\mathcal{X}$, Sym($\mathcal{X}$) is 
the group of all bijective maps from $\mathcal{X}$ to itself. The group operation is composition of maps and the identity element is the identity map on $\mathcal{X}$.
\end{definition}
\begin{definition}[Homomorphism]\label{def: homomorphism}
A map between groups $\Gamma: G \mapsto H$ is called a \emph{(group) homomorphism} if it preserves group multiplication, $\Gamma(g_1g_2) = \Gamma(g_1)\Gamma(g_2)$.
\end{definition}
\begin{definition}[Group action]\label{def: group_action}
An \emph{action} of a group $G$ on a set $\mathcal{X}$ is defined as a homomorphism $\Gamma:G \rightarrow$ Sym$(\mathcal{X})$. We will denote the image of $g\in G$ under $\Gamma$ as $\Gamma_g$.
\end{definition}
A bijective homomorphism from a group to itself is called an automorphism.

\section{Datasets}\label{subsubsec: fc_datasets}
We consider a set of 12 group transformations of which 5 are real group transformations: rotation, horizontal and vertical flips, horizontal and vertical translations, and the remaining 7 are synthetic group symmetries used to scramble the data described in Tab.~\ref{tab: individual_augmentations}. We illustrate samples from G-MNIST and G-Fashion-MNIST with these 12 symmetries applied individually in Fig.~\ref{fig: I_G_MNIST} and Fig.~\ref{fig: I_G_Fasion_MNIST} respectively. Similarly, we also apply a combination of 5 subsets of these transformations shown in Tab.~\ref{tab: augmentations} to obtain the samples from G-MNIST and G-Fashion-MNIST shown in Fig.~\ref{fig: Aug_G_MNIST} and Fig.~\ref{fig: Aug_G_Fashion_MNIST}. The data we are providing does not contain personally identifiable information or offensive content.

\begin{figure*}%
    \centering
    \subfloat[Original]{{\includegraphics[width=3.2cm]{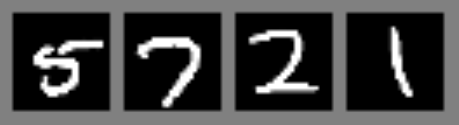} }\label{fig: surprise_figure_k}}\\
    \subfloat[Rotations]{{\includegraphics[width=3.2cm]{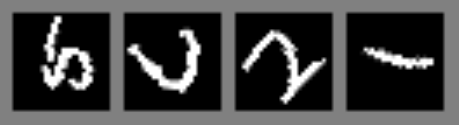} } \label{fig: surprise_figure_p}}%
    \subfloat[Horizontal flips]{{\includegraphics[width=3.2cm]{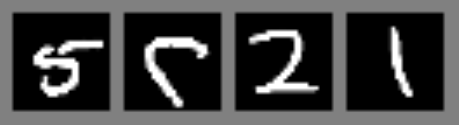} } \label{fig: surprise_figure_t}}%
    \subfloat[Vertical flips]{{\includegraphics[width=3.2cm]{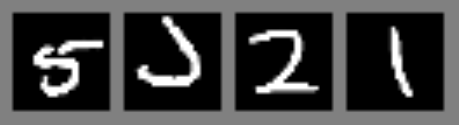} } \label{fig: surprise_figure_t}}
    \subfloat[Horizontal translations]{{\includegraphics[width=3.2cm]{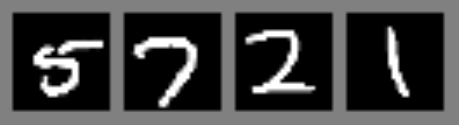} }\label{fig: surprise_figure_k}}\\
    \subfloat[Vertical translations]{{\includegraphics[width=3.2cm]{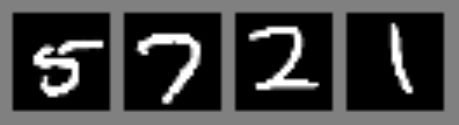} } \label{fig: surprise_figure_p}}%
    \subfloat[Rotation scrambles]{{\includegraphics[width=3.2cm]{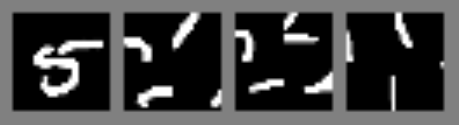} } \label{fig: surprise_figure_t}}%
    \subfloat[Horizontal scrambles]{{\includegraphics[width=3.2cm]{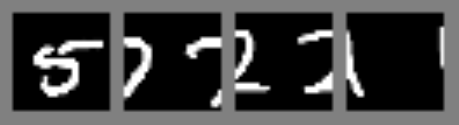} } \label{fig: surprise_figure_t}}%
    \subfloat[Vertical scrambles]{{\includegraphics[width=3.2cm]{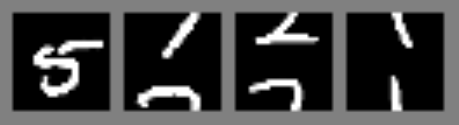} }\label{fig: surprise_figure_k}}\\
    \subfloat[Left vertical scrambles]{{\includegraphics[width=3.2cm]{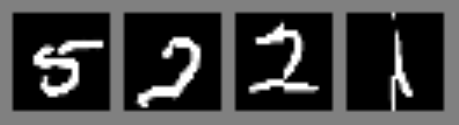} } \label{fig: surprise_figure_p}}%
    \subfloat[Right vertical scrambles]{{\includegraphics[width=3.2cm]{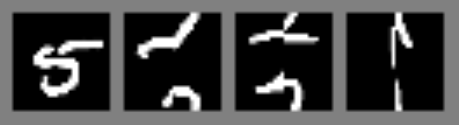} } \label{fig: surprise_figure_t}}%
    \subfloat[Top horizontal scrambles]{{\includegraphics[width=3.2cm]{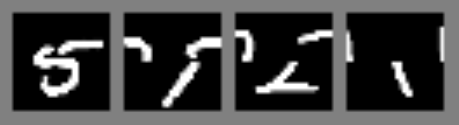} } \label{fig: surprise_figure_t}}%
    \subfloat[Bottom horizontal scrambles]{{\includegraphics[width=3.2cm]{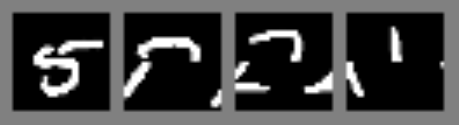} } \label{fig: surprise_figure_t}}%
    \vspace{-3mm}
    \caption{Samples of the digits `5',`7',`2',`1' from the G-MNIST dataset with randomly applied single group transformations.}%
    \label{fig: I_G_MNIST}
\end{figure*}

\begin{figure*}%
    \centering
    \subfloat[Original]{{\includegraphics[width=3.2cm]{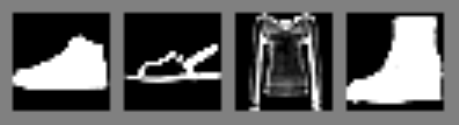} }\label{fig: surprise_figure_k}}\\
    \subfloat[Rotations]{{\includegraphics[width=3.2cm]{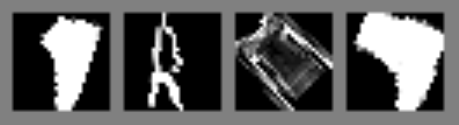} } \label{fig: surprise_figure_p}}%
    \subfloat[Horizontal flips]{{\includegraphics[width=3.2cm]{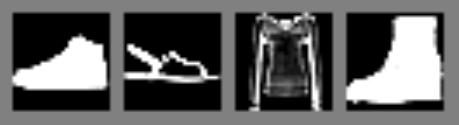} } \label{fig: surprise_figure_t}}%
    \subfloat[Vertical flips]{{\includegraphics[width=3.2cm]{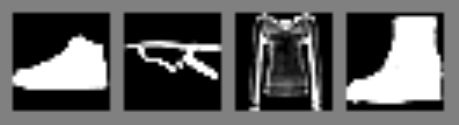} } \label{fig: surprise_figure_t}}
    \subfloat[Horizontal translations]{{\includegraphics[width=3.2cm]{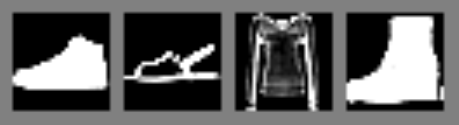} }\label{fig: surprise_figure_k}}\\
    \subfloat[Vertical translations]{{\includegraphics[width=3.2cm]{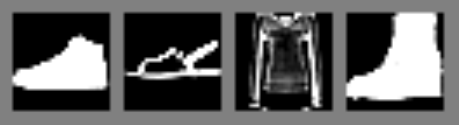} } \label{fig: surprise_figure_p}}%
    \subfloat[Rotation scrambles]{{\includegraphics[width=3.2cm]{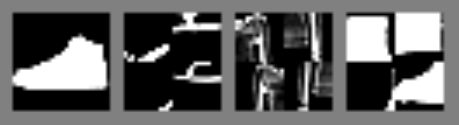} } \label{fig: surprise_figure_t}}%
    \subfloat[Horizontal scrambles]{{\includegraphics[width=3.2cm]{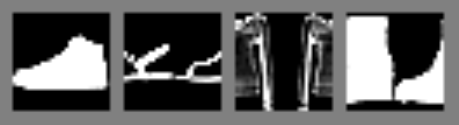} } \label{fig: surprise_figure_t}}%
    \subfloat[Vertical scrambles]{{\includegraphics[width=3.2cm]{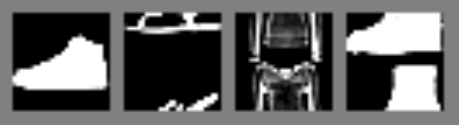} }\label{fig: surprise_figure_k}}\\
    \subfloat[Left vertical scrambles]{{\includegraphics[width=3.2cm]{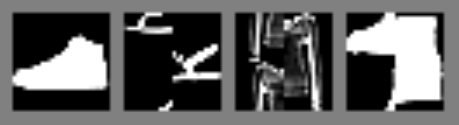} } \label{fig: surprise_figure_p}}%
    \subfloat[Right vertical scrambles]{{\includegraphics[width=3.2cm]{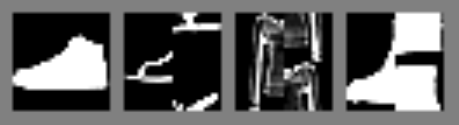} } \label{fig: surprise_figure_t}}%
    \subfloat[Top horizontal scrambles]{{\includegraphics[width=3.2cm]{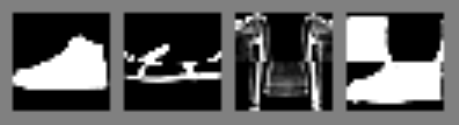} } \label{fig: surprise_figure_t}}%
    \subfloat[Bottom horizontal scrambles]{{\includegraphics[width=3.2cm]{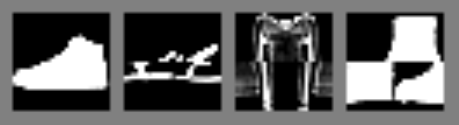} } \label{fig: surprise_figure_t}}%
    \vspace{-3mm}
    \caption{Samples of the items `Sneaker', `Sandal', `Pullover', `Ankle boot' from the G-Fashion-MNIST with randomly applied single group transformations.}%
    \label{fig: I_G_Fasion_MNIST}
\end{figure*}

\begin{table}
\caption{Description of individual transformations. These transformations and their combinations were used to construct the G-MNIST and G-Fashion-MNIST datasets. }
    \centering
    \begin{tabular}{|p{2cm}| p{5cm}|}
    \hline
              \textbf{} & Augmentation descriptions \\
              \hline
        \textbf{IAug0} & Original \\
        \textbf{IAug1} & Rotations \\
        \textbf{IAug2} & Horizontal flips \\
        \textbf{IAug3} & Vertical flips \\
        \textbf{IAug4} & Horizontal translations \\
        \textbf{IAug5} & Vertical translations \\
        \textbf{IAug6} & Rotation scrambles \\
        \textbf{IAug7} & Horizontal scrambles \\
        \textbf{IAug8} & Vertical scrambles \\
        \textbf{IAug9} & Left vertical scrambles \\
        \textbf{IAug10} & Right vertical scrambles \\
        \textbf{IAug11} & Top horizontal scrambles \\
        \textbf{IAug12} & Bottom horizontal scrambles \\
        \hline
    \end{tabular}
    \label{tab: individual_augmentations}
\end{table}

\begin{table}
\centering
\caption{Augmentation arrays showing the groups contained in each augmentation array. Here, 1 in $i$th row and $j$th column indicates the presence of the $j$th group in the array Aug$i$.}
    \centering
    \begin{tabular}{|p{0.7cm}| p{0.17cm} p{0.17cm} p{0.17cm} p{0.17cm} p{0.17cm} p{0.17cm} p{0.17cm} p{0.17cm} p{0.17cm} p{0.17cm} p{0.17cm} p{0.17cm}|}
    \hline
              \textbf{}  & 1 & 2 & 3 & 4 & 5 & 6 & 7 & 8 & 9 & 10 & 11 & 12 \\
              \hline
        \textbf{Aug0} & 0 & 0 & 0 & 0 & 0 & 0 & 0 & 0 & 0 & 0 & 0 & 0 \\
        \textbf{Aug1} & 0 & 1 & 1 & 0 & 0 & 0 & 1 & 0 & 0 & 0 & 0 & 1 \\
        \textbf{Aug2} & 1 & 1 & 0 & 0 & 0 & 0 & 0 & 0 & 1 & 1 & 0 & 0 \\
        \textbf{Aug3} & 1 & 0 & 0 & 1 & 0 & 0 & 0 & 1 & 0 & 0 & 1 & 0 \\
        \textbf{Aug4} & 1 & 1 & 1 & 1 & 1 & 1 & 0 & 0 & 0 & 0 & 0 & 0 \\
        \textbf{Aug5} & 1 & 1 & 1 & 1 & 1 & 1 & 1 & 1 & 1 & 1 & 1 & 1 \\
        \hline
    \end{tabular}
    \label{tab: augmentations}
\end{table}

\begin{figure*}%
    \centering
    \subfloat[Original]{{\includegraphics[width=3.2cm]{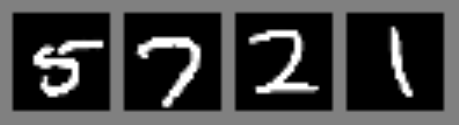} }\label{fig: surprise_figure_k}}
    \subfloat[Aug1]{{\includegraphics[width=3.2cm]{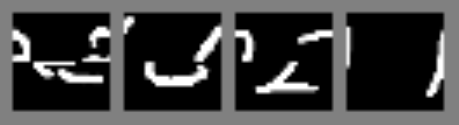} } \label{fig: surprise_figure_p}}%
    \subfloat[Aug2]{{\includegraphics[width=3.2cm]{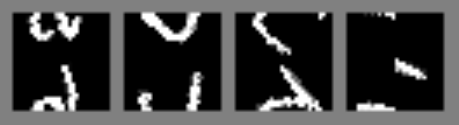} } \label{fig: surprise_figure_t}}%
    \subfloat[Aug3]{{\includegraphics[width=3.2cm]{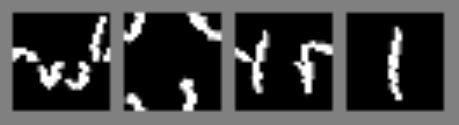} } \label{fig: surprise_figure_t}}\\
    \subfloat[Aug4]{{\includegraphics[width=3.2cm]{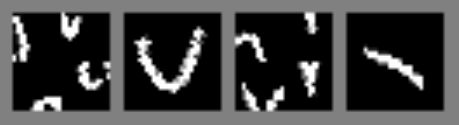} }\label{fig: surprise_figure_k}}
    \subfloat[Aug5]{{\includegraphics[width=3.2cm]{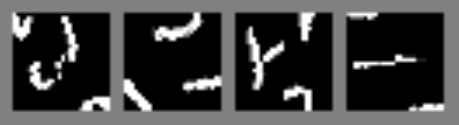} } \label{fig: surprise_figure_p}}%
    \vspace{-3mm}
    \caption{Samples of the digits `5',`7',`2',`1' from the G-MNIST dataset with a combination of group transformations from Tab.~\ref{tab: augmentations} randomly applied.}%
    \label{fig: Aug_G_MNIST}
\end{figure*}

\begin{figure*}%
    \centering
    \subfloat[Original]{{\includegraphics[width=3.2cm]{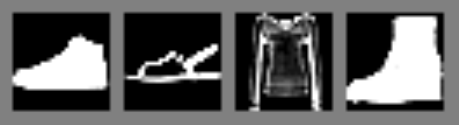} }\label{fig: surprise_figure_k}}
    \subfloat[Aug1]{{\includegraphics[width=3.2cm]{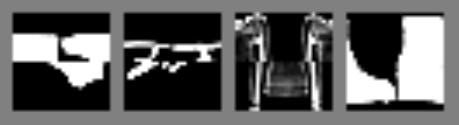} } \label{fig: surprise_figure_p}}%
    \subfloat[Aug2]{{\includegraphics[width=3.2cm]{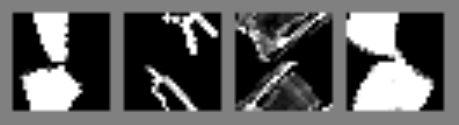} } \label{fig: surprise_figure_t}}%
    \subfloat[Aug3]{{\includegraphics[width=3.2cm]{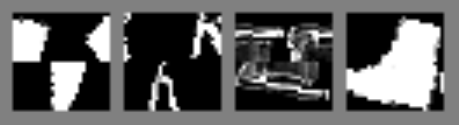} } \label{fig: surprise_figure_t}}\\
    \subfloat[Aug4]{{\includegraphics[width=3.2cm]{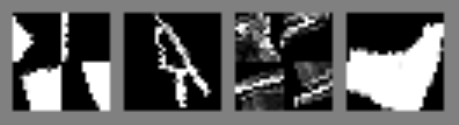} }\label{fig: surprise_figure_k}}
    \subfloat[Aug5]{{\includegraphics[width=3.2cm]{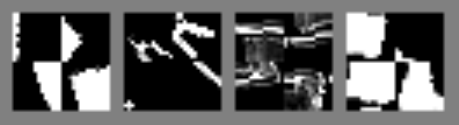} } \label{fig: surprise_figure_p}}%
    \vspace{-3mm}
    \caption{Samples of the items `Sneaker', `Sandal', `Pullover', `Ankle boot' from the G-Fashion-MNIST dataset with a combination of group transformations from Tab.~\ref{tab: augmentations} randomly applied.}%
    \label{fig: Aug_G_Fashion_MNIST}
\end{figure*}
Tab.~\ref{tab: individual_augmentations} provides descriptions for the individual transformations used for constructing the datasets G-MNIST and G-Fashion-MNIST. In Tab.~\ref{tab: individual_augmentations}, the transformations IAug0 to IAug5 are real transformations where as the rest of the them correspond to well defined but synthetic groups. Samples of datasets generated from these groups are shown in Fig.~\ref{fig: I_G_MNIST},~\ref{fig: I_G_Fasion_MNIST}, ~\ref{fig: Aug_G_MNIST}, and ~\ref{fig: Aug_G_Fashion_MNIST}.

Tab.~\ref{tab: MNIST-top5-states} and ~\ref{tab: FashionMNIST-top5-states} shows the top 5 states obtained from the deep Q-learning experiments in Sec.~\ref{subsec: exp_Group_FC_Nets} for G-MNIST and G-Fashion-MNIST datasets respectively.

\begin{table}
\caption{Abbreviations for equivariance and augmentations used for describing top states obtained in deep Q-learning for GCNN. }
    \centering
    \begin{tabular}{|p{2cm}| p{6cm}|}
    \hline
              \textbf{} & Equivariance or Augmentation descriptions \\
              \hline
        \textbf{P4} & 90\textdegree-rotation symmetry equivariance \\
        \textbf{H2} & Horizontal symmetry equivariance  \\
        \textbf{V2} & Vertical symmetry equivariance \\
        \textbf{Zn} & Translation symmetry equivariance\\
        \textbf{Rot (30)} & Random rotation between +30\textdegree and -30 \textdegree \\
        \textbf{HFlip (0.5)} & Horizontal flips with prob 0.5 \\
        \textbf{VFlip (0.5)} & Vertical flips with prob 0.5 \\
        \textbf{Trans (0.2)} & Random translations of 0.2 times the image size  \\
        \textbf{Scale (0.1)} & Random upscaling or downscaling of images by 10\% size \\
        \textbf{Shear (10)} & Random shear transformations of 10\textdegree\\
        \hline
    \end{tabular}
    \label{tab: equivariances_augmentations_GCNN}
\end{table}

\section{Basic Equivariant MLP construction Algorithm}\label{sec: basic_algorithm}
Alg.~2 is the algorithmic version of the equivariant construction algorithm using parameter sharing in \eqref{eqn: single_group_equivariance}.
\begin{algorithm}
\begin{algorithmic}
\FOR{$i \in \mathbf{L}$}
    \IF{$\mathbf{V}[i]<0$}
    $C = C+1$ 
        \FOR{$g \in \mathbf{G}$}
            \STATE $index_g = \Gamma_g(i)$
            \STATE $I[index_g] = C$
        \ENDFOR
        \ENDIF
\ENDFOR
\Return $\mathbf{I}$,C
\end{algorithmic}
 \caption{ Basic Equivariant MLP Construction \cite{RavanbakhshSP2017}}
 \label{alg: basic-equivariant-network-construction}
\end{algorithm}

\section{Training Details}\label{sec: training_details}
Here we provide training details used in experiments in Sec.~\ref{sec: experiments}.
\subsection{Group Equivariant MLPs}\label{sec: training_details_gemlp}
\paragraph{Single Equivariance Testing}
We train the networks for single equivariance testing for datasets of size 10K for 10 epochs using stochastic gradient descent with momentum \cite{Qian1999} with learning rate $10^{-3}$ and report the test accuracies on datasets of size 1K in Tab.~\ref{tab: hypothesis_testing_IAug},~\ref{tab: fashionMNIST_hypothesis_testing_IAug} in Sec.~\ref{sec: additional_experiments}.
\paragraph{Deep Q-learning}
For the deep learning part, the child group equivariant MLPs were trained on mini datasets of size 4K with batchsize 64 for 4 epochs using stochastic gradient descent with momentum \cite{Qian1999} with a learning rate of $10^{-3}$ and tested on mini datasets of size 1K to obtain the reward using \eqref{eqn: fes_reward}.
Once the DQNs are trained, we obtain the top states for each case in the rows of Tab.~\ref{tab: MNIST_hypothesis_testing_Aug} and Tab.~\ref{tab: FashionMNIST_hypothesis_testing_Aug} with maximum rewards. Then we retrain equivariant networks with these states on datasets of size 10K for 10 epochs and report the accuracies obtained in Tab.~\ref{tab: hypothesis_testing_Aug_summary}, Tab.~\ref{tab: MNIST_hypothesis_testing_Aug} and Tab.~\ref{tab: FashionMNIST_hypothesis_testing_Aug} under the column AEN. Further, we also provide the top states obtained in each case and  the number of parameters in them for general interest in Tab.~\ref{tab: MNIST-top5-states} and Tab.~\ref{tab: FashionMNIST-top5-states} for G-MNIST and G-Fashion-MNIST datasets in Sec.~\ref{sec: additional_experiments}, respectively. 
\subsection{Group Equivariant Convolutional Neural Networks}\label{sec: training_details_gcnn}
For the deep Q-learning part, the GCNNs were trained on mini datasets of size 6K with batchsize 32 for 4 epochs using stochastic gradient descent with momentum with a learning rate of $10^{-3}$ and tested on mini datasets of size 1K.

\section{Additional Results}\label{sec: additional_experiments}
Here we provide complete set of experimental results, summary of which is given in Sec.~\ref{sec: experiments}.
\begin{table*}
\caption{Test accuracies for equivariant networks for data with single augmentations from G-MNIST. Transformation IAug$i$ in the row correspond to the single transformations from Tab.~\ref{tab: individual_augmentations} and the columns indicate equivariant networks with Eq$i$ having equivariance to transformation IAug$i$.}
    \centering
    \begin{tabular}{p{1cm}| p{0.5cm} p{0.5cm} p{0.5cm} p{0.5cm} p{0.5cm} p{0.5cm} p{0.5cm} p{0.5cm} p{0.5cm} p{0.5cm} p{0.6cm} p{0.6cm} p{0.6cm}}
              \textbf{}  & Eq0 & Eq1 & Eq2 & Eq3 & Eq4 & Eq5 & Eq6 & Eq7 & Eq8 & Eq9 & Eq10 & Eq11 & Eq12 \\
              \hline
        \textbf{IAug0} & 95.9& \textbf{96.7}& 95.8& 96.4& 83.2& 65.7& 95.8& 95.3& 96.1& 96.4& 96.1& 96.0& 96.6 \\
        \textbf{IAug1} & 83.5& \textbf{87.2}& 85.6& 82.5& 50.1& 49.1& 84.1& 83.9& 81.8& 81.8& 83.2& 83.6& 85.4\\
        \textbf{IAug2} &91.2& 93.0& 93.0& 93.4& 81.9& 59.1& 92.9& 92.6& 92.1& 91.9& 91.5& \textbf{93.8}& 92.9 \\
        \textbf{IAug3} &92.3& 92.7& \textbf{93.2}& 93.2& 79.2& 62.9& 93.0& 91.8& 91.6& 91.9& 92.5& 92.7& 92.9  \\
        \textbf{IAug4} &95.6& 95.9& 96.2& \textbf{96.6}& 81.6& 64.8& 96.3& 95.4& 95.3& 96.4& 95.5& 96.1& 96.4  \\
        \textbf{IAug5} &  95.7& 96.2& 96.0& 96.4& 81.0& 64.8& 96.1& 95.3& 95.6& 96.2& 96.0& 96.0& \textbf{97.0}\\
        \textbf{IAug6} &94.2& 94.4& 93.9& 93.5& 84.8& 72.3& 94.3& 93.8& 93.6& 94.0& 94.7& \textbf{95.3}& 94.8  \\
        \textbf{IAug7} &94.2& 94.7& 94.8& 94.1& 81.8& 62.3& 94.4& 93.9& 93.2& 95.8& 95.4& 94.5& \textbf{95.9} \\
        \textbf{IAug8} &94.4& 94.8& 94.4& 94.1& 77.1& 61.1& \textbf{95.2}& 94.1& 94.8& 93.3& 93.3& 94.7& 94.3  \\
        \textbf{IAug9} & 94.6& 95.2& 93.9& 94.8& 83.7& 64.9& 94.9& \textbf{95.5}& 95.1& 94.2& 94.1& 94.5& 94.7 \\
        \textbf{IAug10} &94.6& 95.6& 95.2& 95.5& 83.5& 63.4& 94.9& 94.5& \textbf{96.4}& 95.3& 95.2& 95.2& 95.4  \\
        \textbf{IAug11} &  95.6& 95.1& 95.6& 94.1& 82.1& 72.0& 95.2& 94.8& 95.1& 95.6& 95.1& 95.2& \textbf{96.3}\\
        \textbf{IAug12} &95.1& 95.4& 95.4& 95.8& 81.5& 71.5& 95.2& 95.2& 95.6& 94.3& \textbf{96.3}& 95.1& 95.9  \\
    \end{tabular}
    \label{tab: hypothesis_testing_IAug}
\end{table*}

\begin{table*}
\caption{Test accuracies for equivariant networks for data with single augmentations from G-Fashion-MNIST. Transformation IAug$i$ in the row correspond to the single transformations from Tab.~\ref{tab: individual_augmentations} and the columns indicate equivariant networks with Eq$i$ having equivariance to transformation IAug$i$.}
    \centering
    \begin{tabular}{p{1cm}| p{0.5cm} p{0.5cm} p{0.5cm} p{0.5cm} p{0.5cm} p{0.5cm} p{0.5cm} p{0.5cm} p{0.5cm} p{0.5cm} p{0.6cm} p{0.6cm} p{0.6cm}}
              \textbf{}  & Eq0 & Eq1 & Eq2 & Eq3 & Eq4 & Eq5 & Eq6 & Eq7 & Eq8 & Eq9 & Eq10 & Eq11 & Eq12 \\
              \hline
        \textbf{IAug0} &84.8& \textbf{85.9}& 84.6& 84.8& 73.4& 76.4& 85.5& 84.1& 85.4& 85.2& 84.1& 85.4& 85.5 \\
        \textbf{IAug1} &66.3& \textbf{72.0}& 69.9& 68.9& 49.1& 48.0& 66.1& 65.8& 66.4& 64.7& 66.0& 66.2& 66.3\\
        \textbf{IAug2} &83.8& 84.3& \textbf{85.4}& 84.5& 73.5& 74.6& 84.7& 83.9& 84.0& 82.7& 83.3& 84.2& 83.9\\
        \textbf{IAug3} &81.8& 82.8& 82.6& \textbf{84.0}& 71.5& 74.2& 82.7& 82.2& 81.8& 82.5& 81.9& 82.1& 83.2\\
        \textbf{IAug4} &84.8& 85.0& 84.3& \textbf{85.2}& 72.6& 72.7& 84.9& 84.2& 83.6& 83.7& 84.3& 83.3& 84.6\\
        \textbf{IAug5} &84.3& 84.3& 83.9& 84.2& 71.6& 74.4& \textbf{85.0}& 84.0& 83.4& 84.4& 83.9& 84.5& 84.9\\
        \textbf{IAug6} &84.6& 83.3& 84.6& \textbf{84.9}& 71.6& 76.9& 84.2& 84.0& 83.8& 84.3& 83.7& 84.5& 84.0\\
        \textbf{IAug7} &83.9& 84.4& 83.3& 85.1& 70.3& 72.6& 84.7& 84.8& 84.1& 84.4& 84.6& 84.7& \textbf{86.8}\\
        \textbf{IAug8} &82.6& 84.1& 83.1& 83.8& 69.5& 75.4& \textbf{85.3}& 85.2& 84.8& 82.7& 82.8& 84.2& 83.1\\
        \textbf{IAug9} &84.7& 84.0& 84.2& \textbf{85.2}& 74.8& 76.0& 85.0& 84.7& 83.7& 84.8& 83.9& 84.2& 84.8\\
        \textbf{IAug10} &\textbf{84.6}& 84.2& 83.2& 83.8& 72.3& 74.8& 84.3& 83.7& 84.0& 84.3& 83.5& 83.7& 83.9\\
        \textbf{IAug11} &\textbf{85.7}& 84.2& 84.4& 84.6& 72.8& 74.6& 85.6& 85.5& 84.7& 84.6& 84.6& 84.7& 84.8\\
        \textbf{IAug12} &84.3& 84.4& 84.3& \textbf{85.0}& 73.8& 75.8& 84.9& 84.4& 83.9& 84.3& 84.1& 84.7& 84.5\\
    \end{tabular}
    \label{tab: fashionMNIST_hypothesis_testing_IAug}
\end{table*}

\begin{table*}
\caption{Test accuracies for equivariant networks for data with multiple augmentations from G-MNIST. Transformation Aug$i$ in the row correspond to the array of transformations from Tab.~\ref{tab: augmentations} and the columns indicate equivariant networks with Eq$i$ having equivariance to transformation IAug$i$ from Tab.~\ref{tab: individual_augmentations}.}
    \centering
    \begin{tabular}{p{0.6cm}| p{0.5cm} p{0.5cm} p{0.5cm} p{0.5cm} p{0.5cm} p{0.5cm} p{0.5cm} p{0.5cm} p{0.5cm} p{0.5cm} p{0.5cm} p{0.5cm} p{0.5cm} p{0.5cm}}
              \textbf{}  & Eq0 & Eq1 & Eq2 & Eq3 & Eq4 & Eq5 & Eq6 & Eq7 & Eq8 & Eq9 & Eq10 & Eq11 & Eq12 & AEN \\
              \hline
        \textbf{Aug0}&95.9& \textbf{96.7}& 95.8& 96.4& 83.2& 65.7& 95.8& 95.3& 96.1& 96.4& 96.1& 96.0& 96.6& 96.0 \\
        \textbf{Aug1}& 85.7 & 86.7& 87.5& 88.4& 78.2& 54.2& 86.5& 84.1& 84.3& 87.1& 87.7& 83.4& 87.1&\textbf{91.4} \\
        \textbf{Aug2}&66.9& 71.9& 72.0& 70.6& 39.9& 48.1& 71.7& 70.4& 71.5& 65.0& 67.5& 66.4& 67.3& \textbf{77.5}\\
        \textbf{Aug3}&73.7& 73.9& 73.9& 72.1& 46.3& 41.4& 72.7& 69.4& 70.1& 75.3& 68.1& 71.8& 73.2& \textbf{77.4}\\
        \textbf{Aug4}&69.5& \textbf{73.9}& 69.2& 68.6& 44.2& 44.3& 67.7& 66.8& 66.1& 67.1& 66.6& 69.8& 66.8& 72.9\\
        \textbf{Aug5}&51.0& 57.6& 54.4& 55.8& 38.6& 37.8& 54.8& 54.3& 53.9& 54.8& 55.2& 62.2& 56.2& \textbf{69.4}\\
    \end{tabular}
    \label{tab: MNIST_hypothesis_testing_Aug}
\end{table*}

\begin{table*}
\caption{Test accuracies for equivariant networks for data with multiple augmentations from G-Fashion-MNIST. Transformation Aug$i$ in the row correspond to the array of transformations from Tab.~\ref{tab: augmentations} and the columns indicate equivariant networks with Eq$i$ having equivariance to transformation IAug$i$ from Tab.~\ref{tab: individual_augmentations}.}
    \centering
    \begin{tabular}{p{0.65cm}| p{0.5cm} p{0.5cm} p{0.5cm} p{0.5cm} p{0.5cm} p{0.5cm} p{0.5cm} p{0.5cm} p{0.5cm} p{0.5cm} p{0.5cm} p{0.5cm} p{0.5cm} p{0.5cm}}
              \textbf{}  & Eq0 & Eq1 & Eq2 & Eq3 & Eq4 & Eq5 & Eq6 & Eq7 & Eq8 & Eq9 & Eq10 & Eq11 & Eq12 & AEN \\
              \hline
        \textbf{Aug0}  &84.8& \textbf{85.7}& 84.6& 84.8& 73.4& 76.4& 85.5& 84.1& 85.4& 85.2& 84.1& 85.4& 85.5 & \textbf{85.7} \\
        \textbf{Aug1}  &77.3& 78.5& 78.0& 79.8& 65.6& 68.5& 78.5& 77.4& 77.7& 78.7& 79.6& 77.9& 79.4 & \textbf{80.3}\\
        \textbf{Aug2}  & 60.1& 61.8& 60.1& 63.8& 41.1& 47.2& 63.0& 60.9& 61.9& 59.5& 58.3& 60.7& 59.6 & \textbf{64.4}\\
        \textbf{Aug3}  &58.0& 59.8& 58.2& 57.8& 43.5& 41.2& 58.0& 59.8& 58.8& 60.0& 57.7& 59.7& 57.1 &\textbf{63.4}\\
        \textbf{Aug4}  &59.4& 60.3& 60.5& 59.1& 40.8& 39.2& 59.8& 60.1& 58.7& 60.1& 58.1& 59.8& 60.1 &\textbf{60.9}\\
        \textbf{Aug5}  &50.8& 54.6& 53.1& 52.5& 36.6& 34.6& 53.8& 54.1& 53.3& 53.5& 52.6& 57.8& 52.3 & \textbf{60.4}\\
    \end{tabular}
    \label{tab: FashionMNIST_hypothesis_testing_Aug}
\end{table*}

\begin{table*}
\caption{Top group symmetries obtained from deep Q-learning corresponding to different combination of transformations in the G-MNIST dataset. Each row correspond to an array of group symmetry, where a 1 in the $i$th column indicates the presence of equivariance to IAug$i$ and $0$ indicates the absence of it.}
    \centering
    \begin{tabular}{|p{0.7cm}| p{0.15cm} p{0.15cm} p{0.15cm} p{0.15cm} p{0.15cm} p{0.15cm} p{0.15cm} p{0.15cm} p{0.15cm} p{0.15cm} p{0.15cm} p{0.15cm}|p{3.5cm}|}
    \hline
              \textbf{}  & 1 & 2 & 3 & 4 & 5 & 6 & 7 & 8 & 9 & 10 & 11 & 12 & \# parameters (millions)\\
              \hline
        \textbf{Aug0} & 0& 1& 0& 0& 0& 0& 0& 0& 0& 1& 0& 0& 0.15 \\
        \textbf{Aug1} & 0& 1& 1& 0& 0& 0& 0& 0& 0& 1& 0& 1& 0.04 \\
        \textbf{Aug2} & 0& 1& 1& 0& 0& 0& 1& 0& 0& 0& 0& 0& 0.04 \\
        \textbf{Aug3} & 1& 0& 0& 0& 0& 0& 0& 1& 1& 0& 1& 1& 0.018 \\
        \textbf{Aug4} & 1& 0& 0& 0& 0& 0& 0& 0& 0& 0& 1& 0& 0.03 \\
        \textbf{Aug5} & 1& 0& 0& 0& 0& 1& 0& 1& 0& 0& 1& 0& 0.018 \\
        \hline
    \end{tabular}
    \label{tab: MNIST-top5-states}
\end{table*}

\begin{table*}
\caption{Top group symmetries obtained from deep Q-learning corresponding to different combination of transformations in the G-Fashion-MNIST dataset. Each row correspond to an array of group symmetry, where a 1 in the $i$th column indicates the presence of equivariance to IAug$i$ and $0$ indicates the absence of it.}
    \centering
    \begin{tabular}{|p{0.7cm}| p{0.15cm} p{0.15cm} p{0.15cm} p{0.15cm} p{0.15cm} p{0.15cm} p{0.15cm} p{0.15cm} p{0.15cm} p{0.15cm} p{0.15cm} p{0.15cm}|p{3.5cm}|}
    \hline
              \textbf{}  & 1 & 2 & 3 & 4 & 5 & 6 & 7 & 8 & 9 & 10 & 11 & 12 & \# parameters (millions)\\
              \hline
        \textbf{Aug0} & 1& 0& 0& 0& 0& 0& 0& 0& 0& 0& 0& 0& 0.12 \\
        \textbf{Aug1} & 0& 0& 1& 0& 0& 0& 0& 0& 0& 0& 0& 1& 0.12 \\
        \textbf{Aug2} & 0& 1& 0& 0& 0& 1& 0& 0& 0& 0& 0& 0& 0.12 \\
        \textbf{Aug3} & 0& 0& 1& 0& 0& 1& 0& 0& 0& 1& 0& 0& 0.04 \\
        \textbf{Aug4} & 1& 1& 0& 0& 0& 0& 0& 0& 0& 0& 0& 0& 0.06 \\
        \textbf{Aug5} & 1& 0& 1& 0& 0& 1& 0& 0& 0& 0& 1& 0& 0.015 \\
                      \hline
    \end{tabular}
    \label{tab: FashionMNIST-top5-states}
\end{table*}

\begin{table*}
\caption{Top states obtained from deep Q-learning corresponding to equivariances, augmentations, and size used in training a GCNN on several image datasets and comparison with traditional translational (Zn) equivariant CNNs.}
    \centering
    \begin{tabular}{|p{1.9cm}| p{2.5cm}| p{4cm}| p{1cm}| p{2.2cm}|}
    \hline

        \textbf{Dataset} & Equivariances & Augmentations & Size & Test accuracies \\
        \hline
                         & (P4, H2, V2, Zn) & (HFlip (0.5), Scale (0.1))  & Large & $65.25\%$\\
                         & (P4, H2, V2, Zn) & --- & Large & $63.06\%$\\
        \textbf{CIFAR10} & (P4, H2, V2, Zn) & (HFlip (0.5)) & Large & $63.84\%$\\
                         & (P4, H2, V2, Zn) & (Shear (10)) & Large & $62.55\%$\\
                         & (Zn) & --- & Large & $55.17\%$ \\
        \hline
                         & (P4, H2, V2, Zn) & (Scale (0.1), Shear (10)) & Small & $87.91\%$\\
                         & (P4, H2, V2, Zn) & --- & Large & $89.70\%$ \\
        \textbf{SVHN}    & (P4, H2, V2, Zn) & (Shear (10)) & Small & $87.81\%$ \\
                         & (P4, H2, V2, Zn) & (Rot (30), Scale (0.1)) & Large & $87.66\%$ \\
                         & (Zn) & --- & Large & $85.72\%$ \\
        \hline
                         & (P4, H2, V2, Zn) & (Rot (30)) & Small & $93.45\%$ \\
                         & (P4, H2, V2, Zn) & --- & Small & $94.05\%$ \\
        \textbf{RotMNIST}& (P4, H2, V2, Zn) & (Shear (10)) & Small & $94.44\%$ \\
                         & (P4, H2, V2, Zn) & (Scale (0.1)) & Small & $94.57\%$ \\
                         & (Zn) & --- & Large & $92.60\%$ \\
        \hline
                         & (P4, H2, V2, Zn) & --- & Small & $92.74\%$ \\
                         & (P4, H2, V2, Zn) & (Scale (0.1)) & Small & $94.21\%$ \\
        \textbf{ASL}     & (P4, H2, V2, Zn) & (Shear (10)) & Small & $90.93\%$ \\
                         & (P4, H2, V2, Zn) & (VFlip (0.5)) & Large & $93.41\%$ \\
                         & (Zn) & --- & Large & $89.33\%$ \\
        \hline
                         & (P4, H2, V2, Zn) & --- & Large & $91.51\%$ \\
                         & (P4, H2, V2, Zn) & --- & Small & $92.26\%$ \\
        \textbf{EMNIST}  & (P4, H2, V2, Zn) & (Rot (30)) & Large & $91.01\%$ \\
                         & (P4, V2, Zn) & (Shear (10)) & Large & $93.99\%$ \\
                         & (Zn) & --- & Large & $92.93\%$ \\
        \hline
                         & (P4, H2, V2, Zn) & (Shear (10)) & Large & $93.59\%$ \\
                         & (P4, H2, V2, Zn) & --- & Large & $93.57\%$ \\
        \textbf{KMNIST}  & (P4, H2, V2, Zn) & (Rot (30), Scale (0.1)) & Large & $88.30\%$ \\
                         & (P4, H2, V2, Zn) & --- & Small & $93.55\%$ \\
                         & (Zn) & --- & Large & $90.94\%$ \\
        \hline
    \end{tabular}
    \label{tab: gnas_gcnn_results}
\end{table*}

\section{Additional Plots}\label{sec: additional_plots}
Here we provide complete set of plots, selected ones are shown in Sec.~\ref{sec: experiments} to summarize and explain the results.
\begin{figure*}%
    \centering
    \subfloat[Aug0]{{\includegraphics[width=4.5cm]{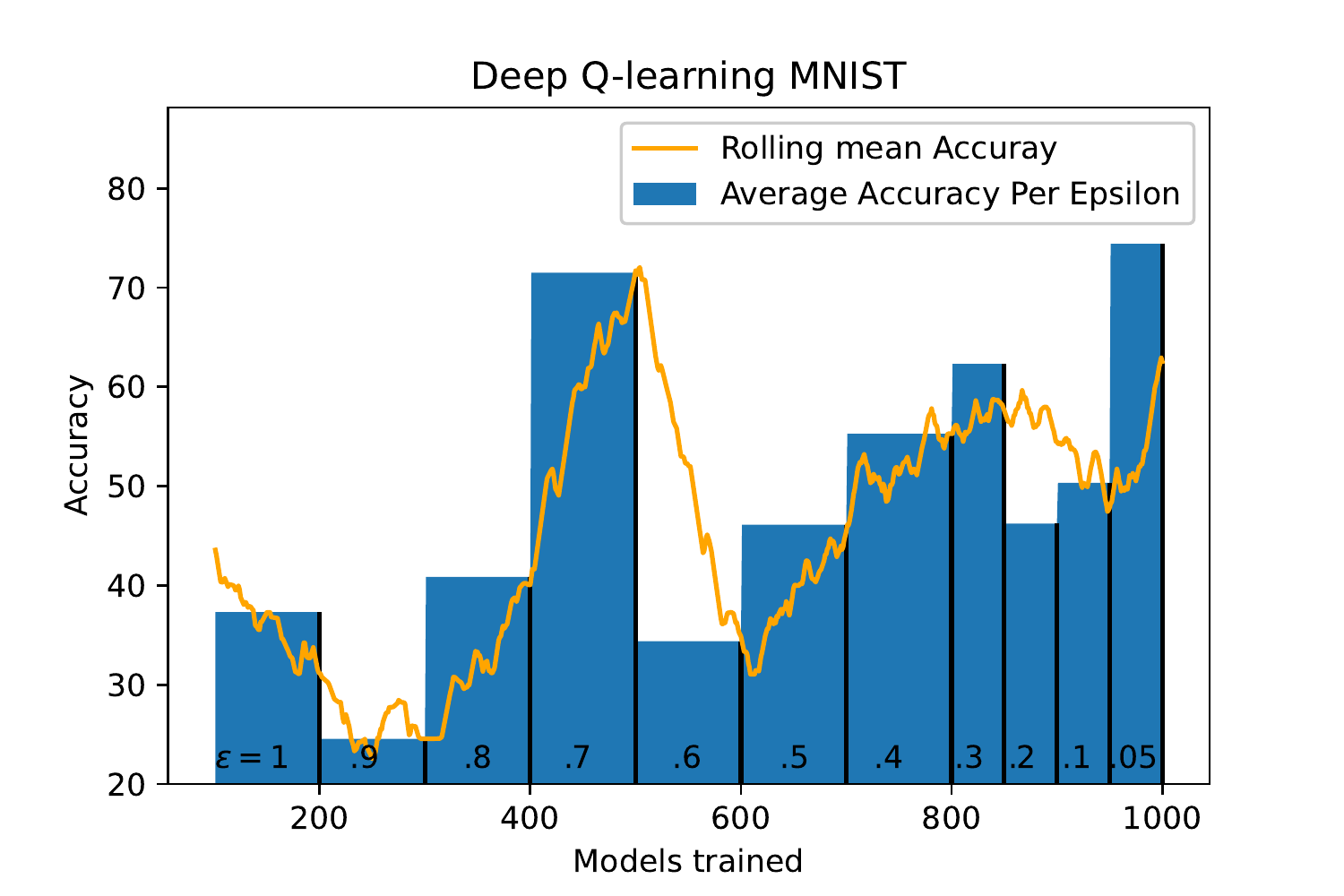} }\label{fig: DQN_MNIST_Aug0}}
    \subfloat[Aug1]{{\includegraphics[width=4.5cm]{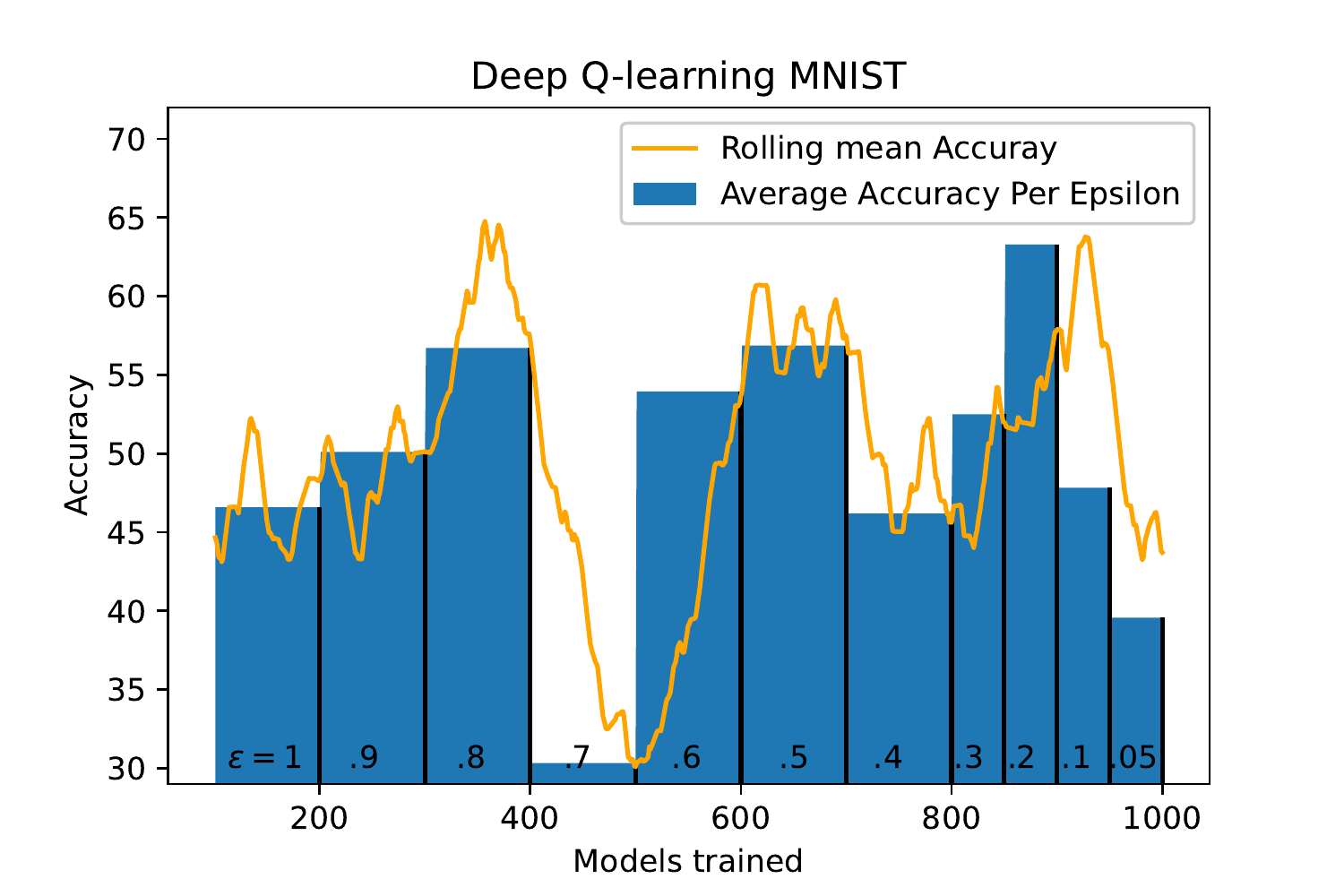} }\label{fig: DQN_MNIST_Aug1}}
    \subfloat[Aug2]{{\includegraphics[width=4.5cm]{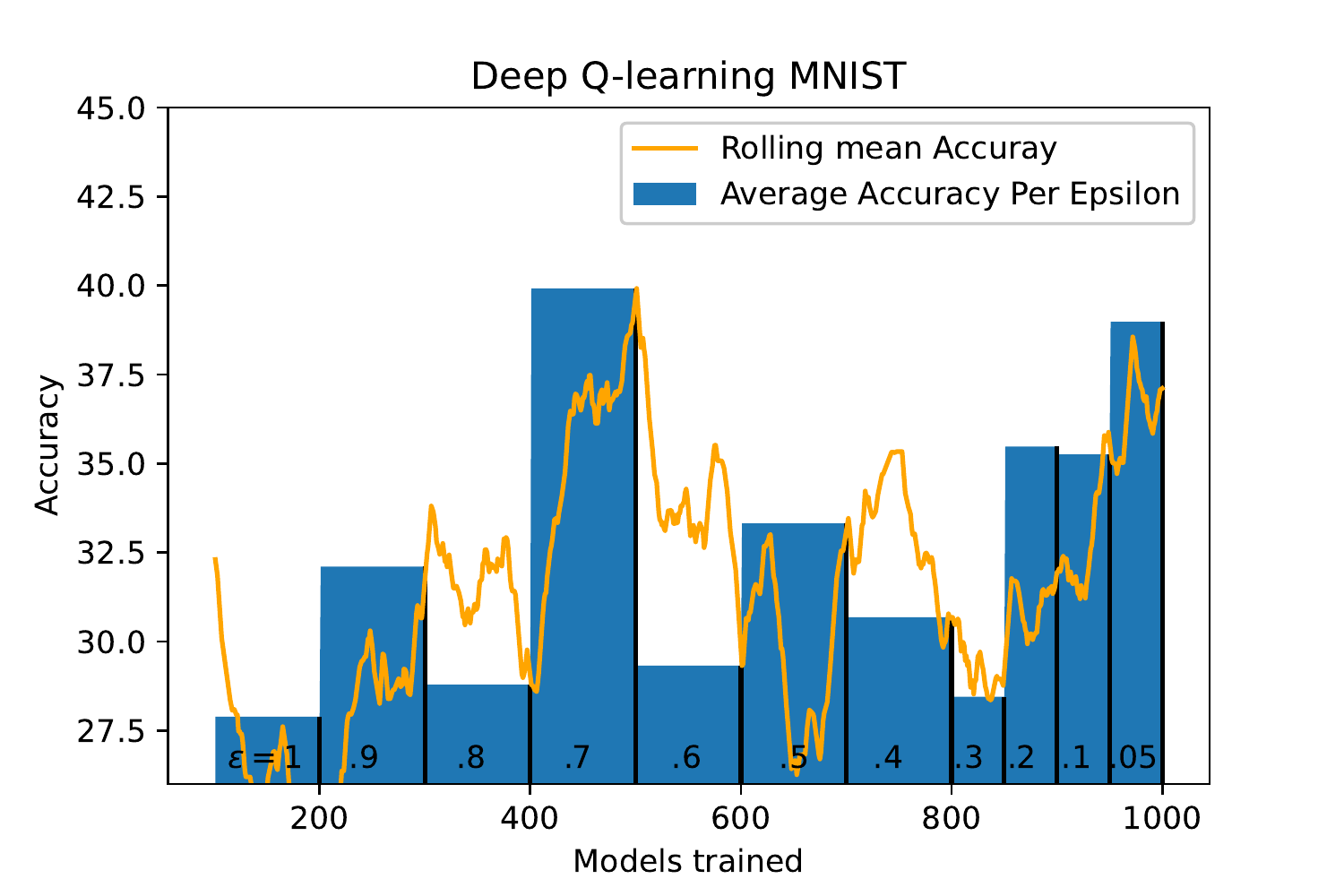} }\label{fig: DQN_MNIST_Aug2}}\\
    \subfloat[Aug3]{{\includegraphics[width=4.5cm]{Images/GNAS_FCNN_MNIST_Aug3_final-eps-converted-to} }\label{fig: DQN_MNIST_Aug3}}
    \subfloat[Aug4]{{\includegraphics[width=4.5cm]{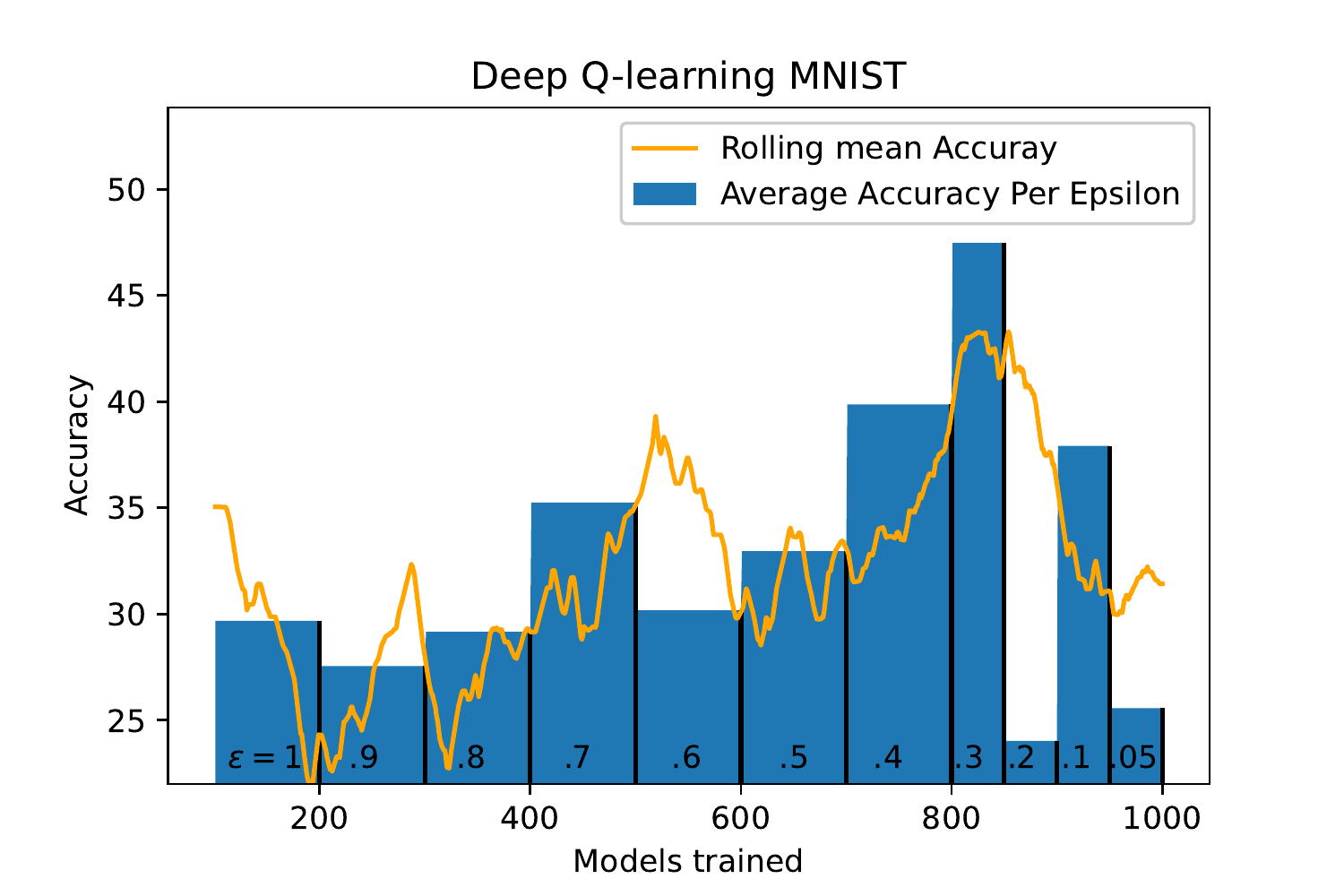} }\label{fig: DQN_MNIST_Aug4}}
    \subfloat[Aug5]{{\includegraphics[width=4.5cm]{Images/GNAS_FCNN_MNIST_Aug5_final-eps-converted-to} }\label{fig: DQN_MNIST_Aug5}}
    \caption{Deep Q-learning performance for datasets with various combinations of transformations from G-MNIST. The Q-learning agent samples group symmetries, which is used for creating equivariant network. A function of the accuracy is used as the reward function. The Q-learning agent shows improvement in accuracy as the $\epsilon$ decreases.}%
    \label{fig: DQN_MNIST}
\end{figure*}

\begin{figure*}%
    \centering
    \subfloat[Aug0]{{\includegraphics[width=4.5cm]{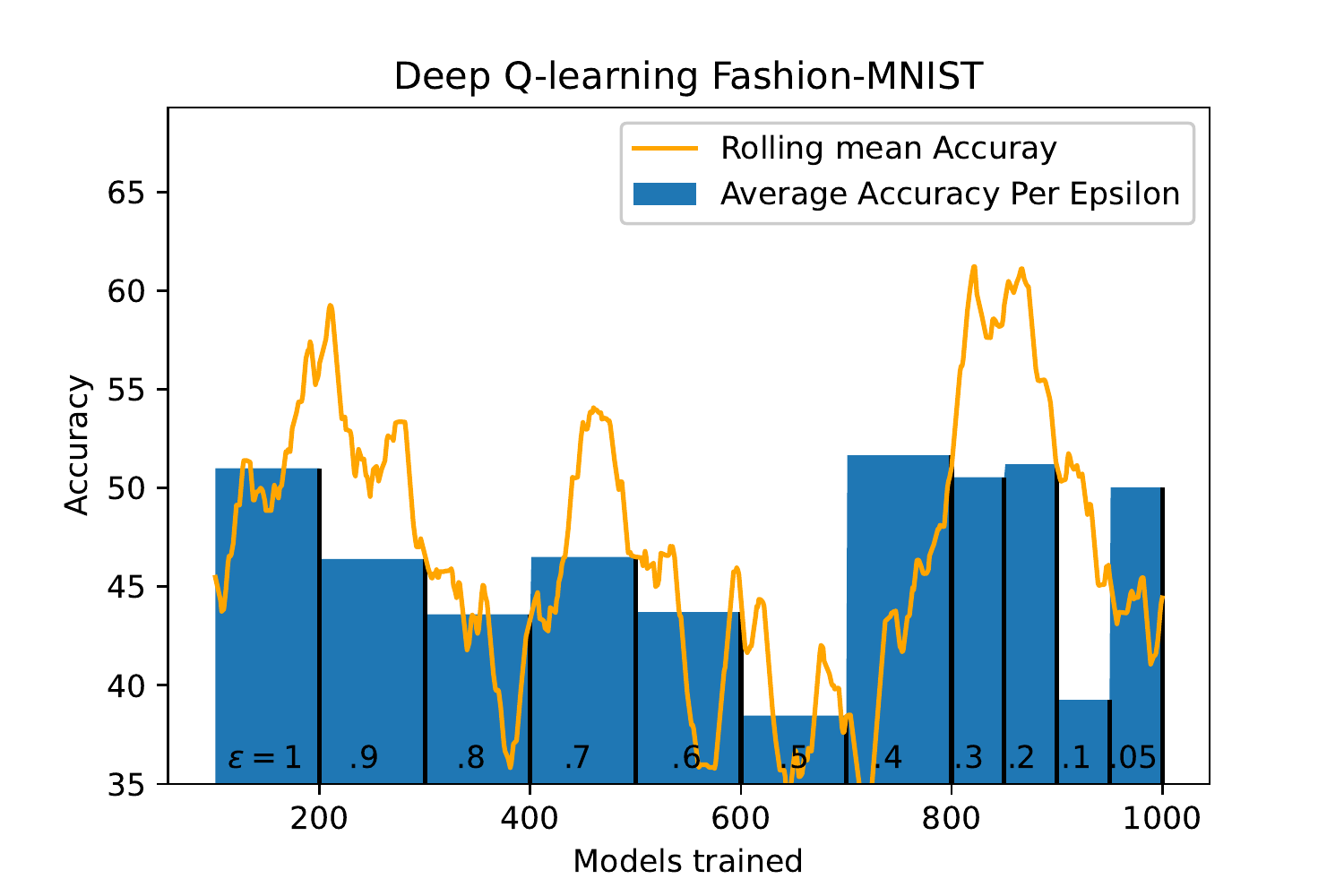} }\label{fig: DQN_FashionMNIST_Aug0}}
    \subfloat[Aug1]{{\includegraphics[width=4.5cm]{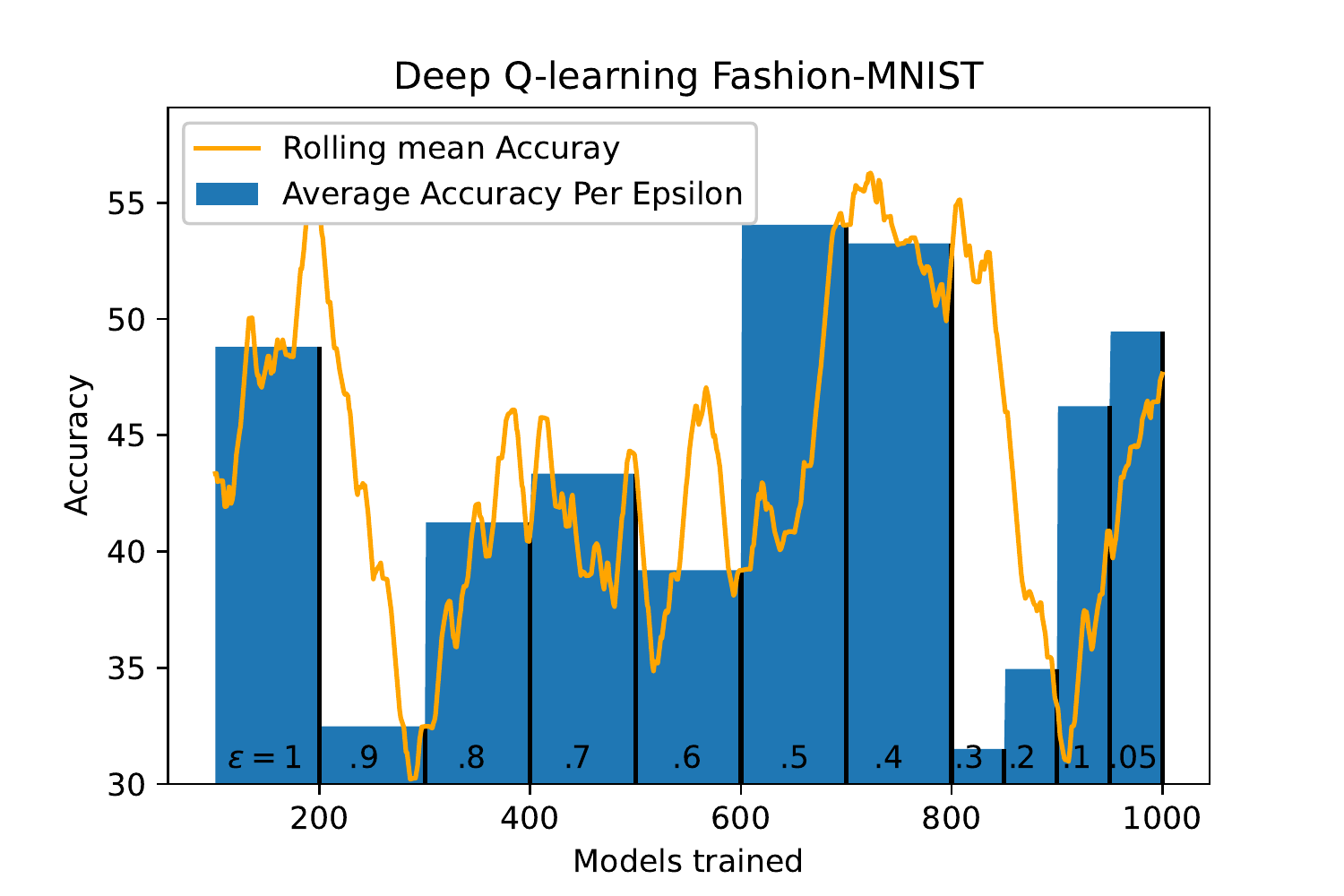} }\label{fig: DQN_FashionMNIST_Aug1}}
    \subfloat[Aug2]{{\includegraphics[width=4.5cm]{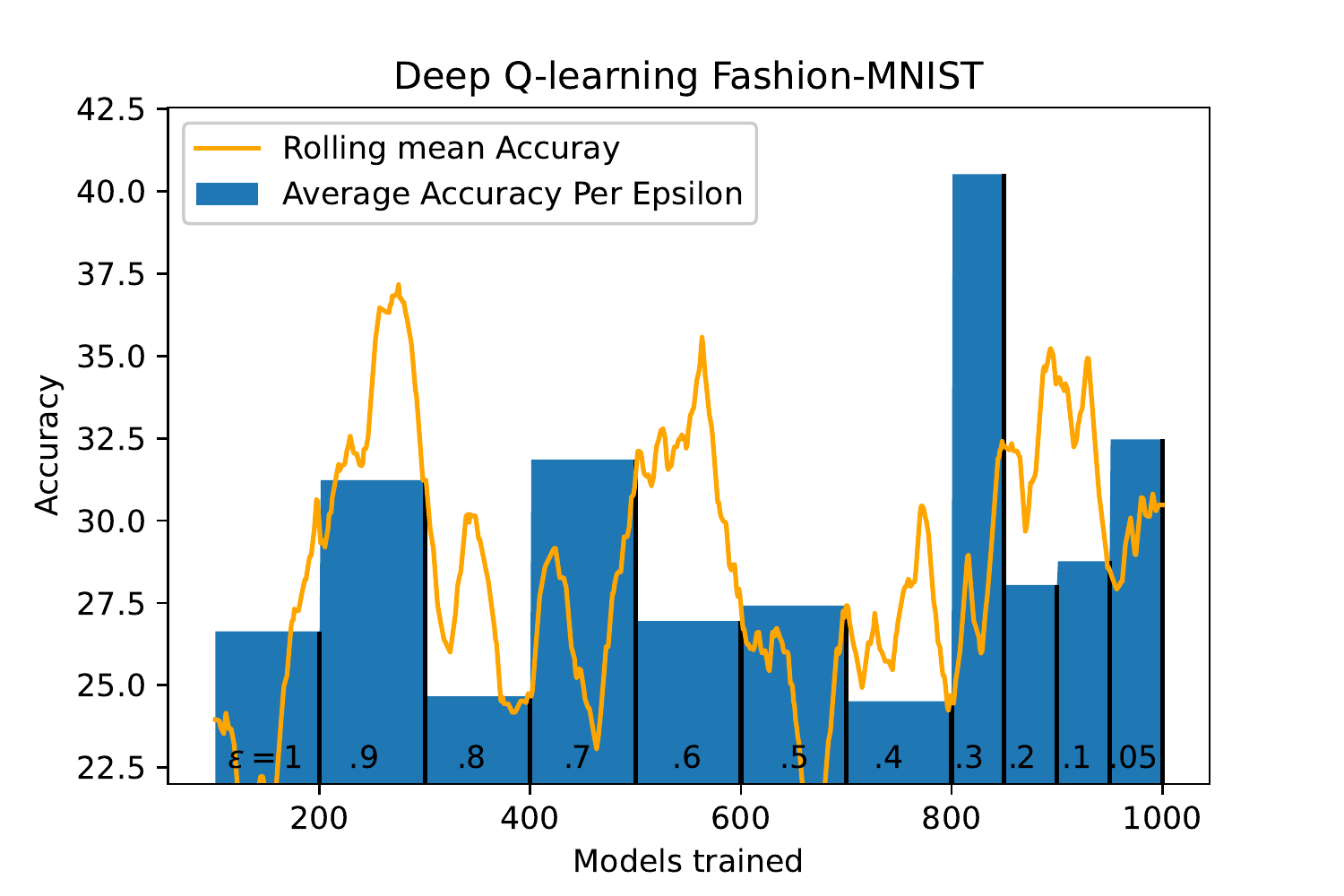} }\label{fig: DQN_FashionMNIST_Aug2}}\\
    \subfloat[Aug3]{{\includegraphics[width=4.5cm]{Images/GNAS_FCNN_FashionMNIST_Aug3_final-eps-converted-to} }\label{fig: DQN_FashionMNIST_Aug3}}
    \subfloat[Aug4]{{\includegraphics[width=4.5cm]{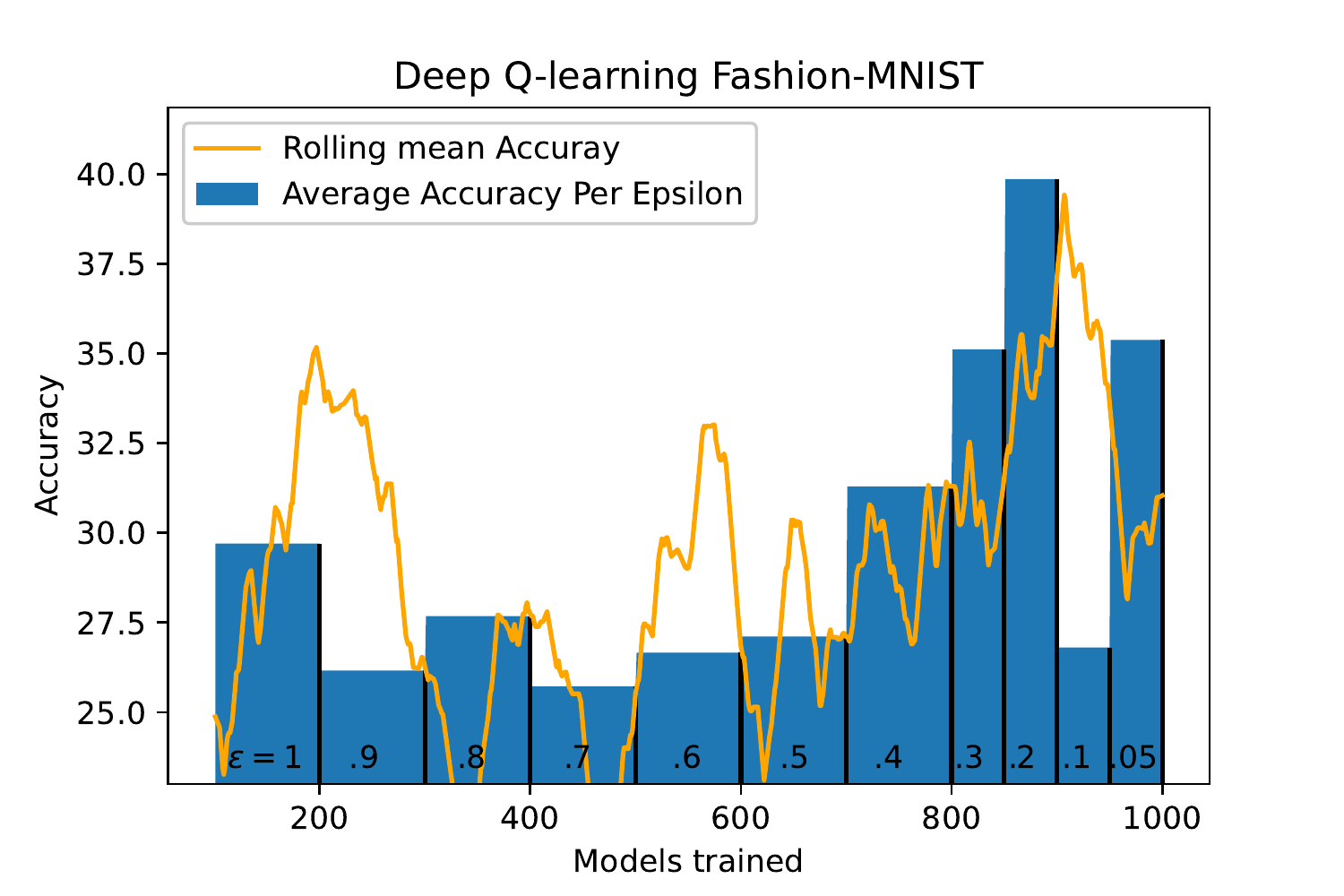} }\label{fig: DQN_FashionMNIST_Aug4}}
    \subfloat[Aug5]{{\includegraphics[width=4.5cm]{Images/GNAS_FCNN_FashionMNIST_Aug5_final-eps-converted-to} }\label{fig: DQN_FashionMNIST_Aug5}}
    \caption{Deep Q-learning performance for datasets with various combinations of transformations from G-Fashion-MNIST. The Q-learning agent samples group symmetries, which is used for creating equivariant network. A function of the accuracy is used as the reward function. The Q-learning agent shows improvement in accuracy as $\epsilon$ decreases.}%
    \label{fig: DQN_FashionMNIST}
\end{figure*}

\begin{figure*}%
    \centering
    \subfloat[CIFAR10]{{\includegraphics[width=4.5cm]{Images/GNAS_GCNN_CIFAR10_summary-eps-converted-to} }\label{fig: DQN_CIFAR10}}
    \subfloat[SVHN]{{\includegraphics[width=4.5cm]{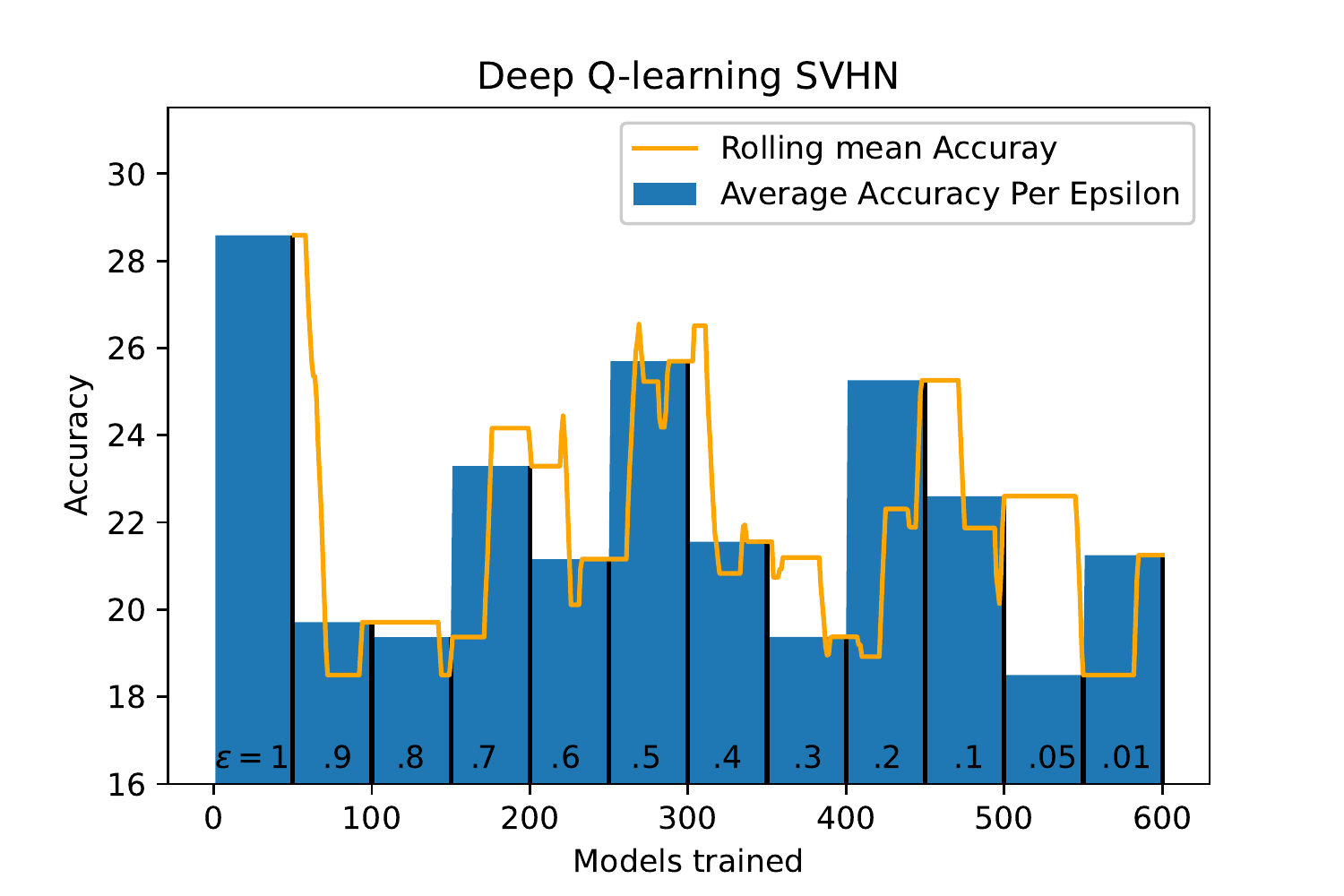} }\label{fig: DQN_SVHN}}
    \subfloat[RotMNIST]{{\includegraphics[width=4.5cm]{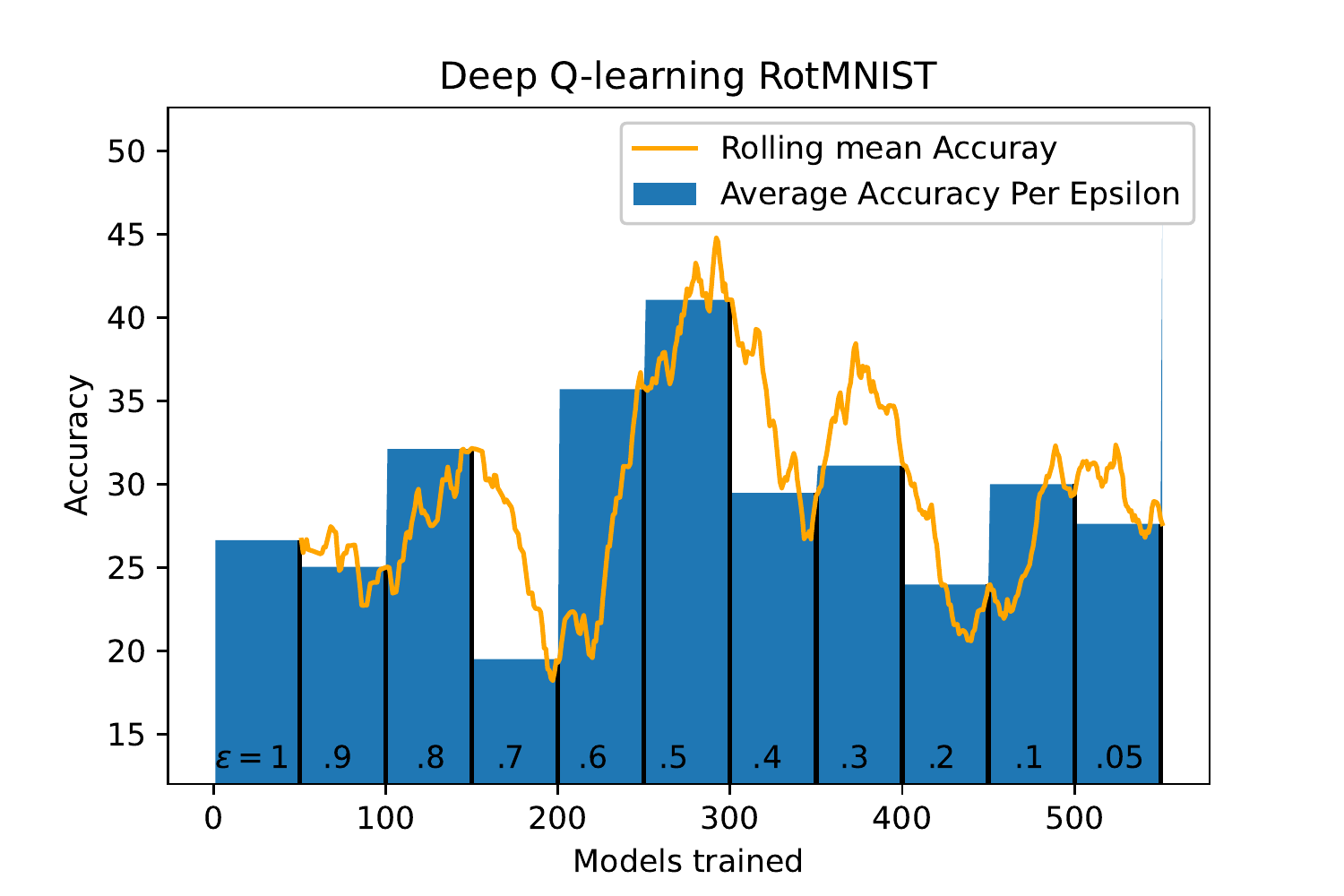} }\label{fig: DQN_RotMNIST}}\\
    \subfloat[ASL]{{\includegraphics[width=4.5cm]{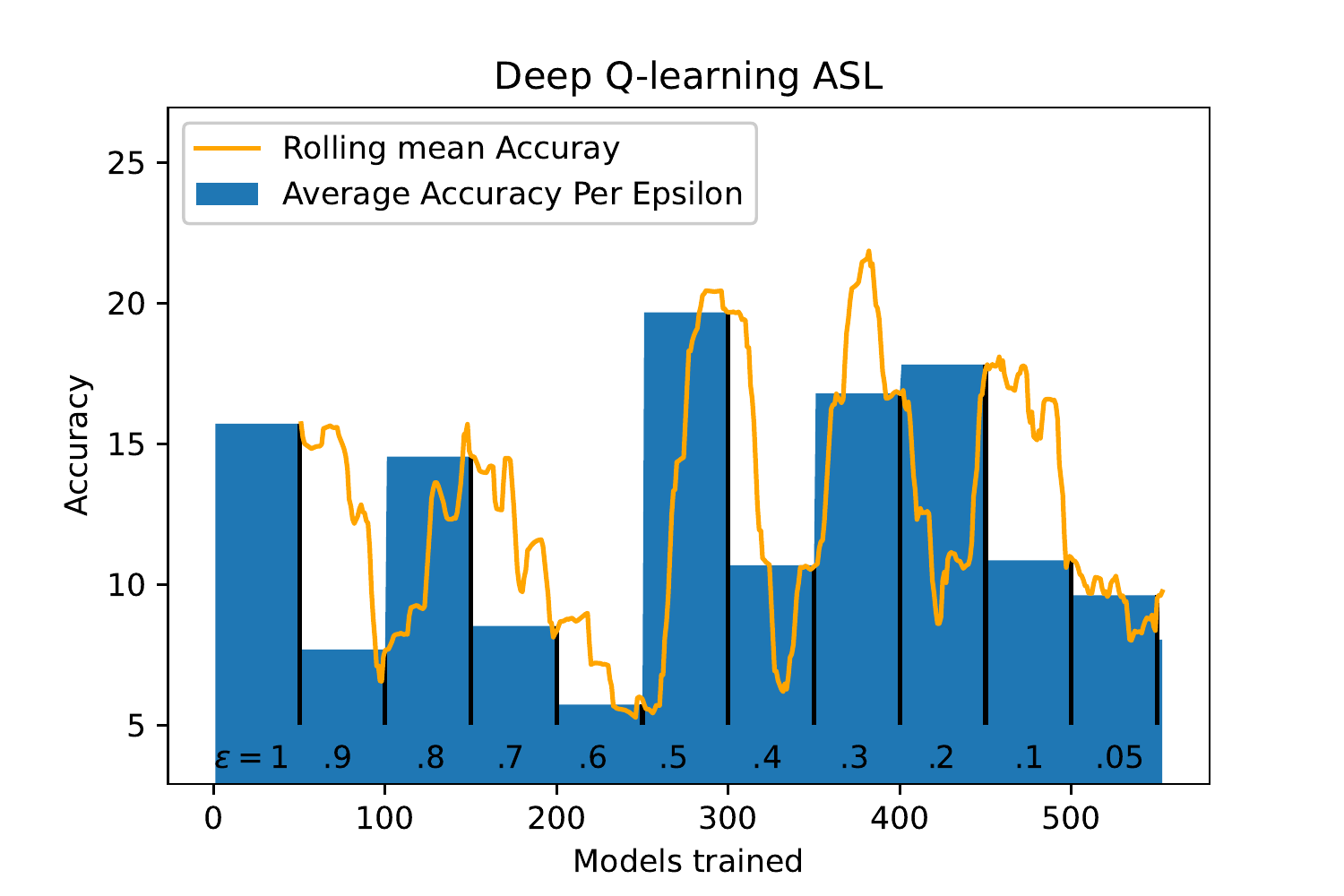} }\label{fig: DQN_ASL}}
    \subfloat[EMNIST]{{\includegraphics[width=4.5cm]{Images/GNAS_GCNN_EMNIST_summary-eps-converted-to} }\label{fig: DQN_EMNIST}}
    \subfloat[KMNIST]{{\includegraphics[width=4.5cm]{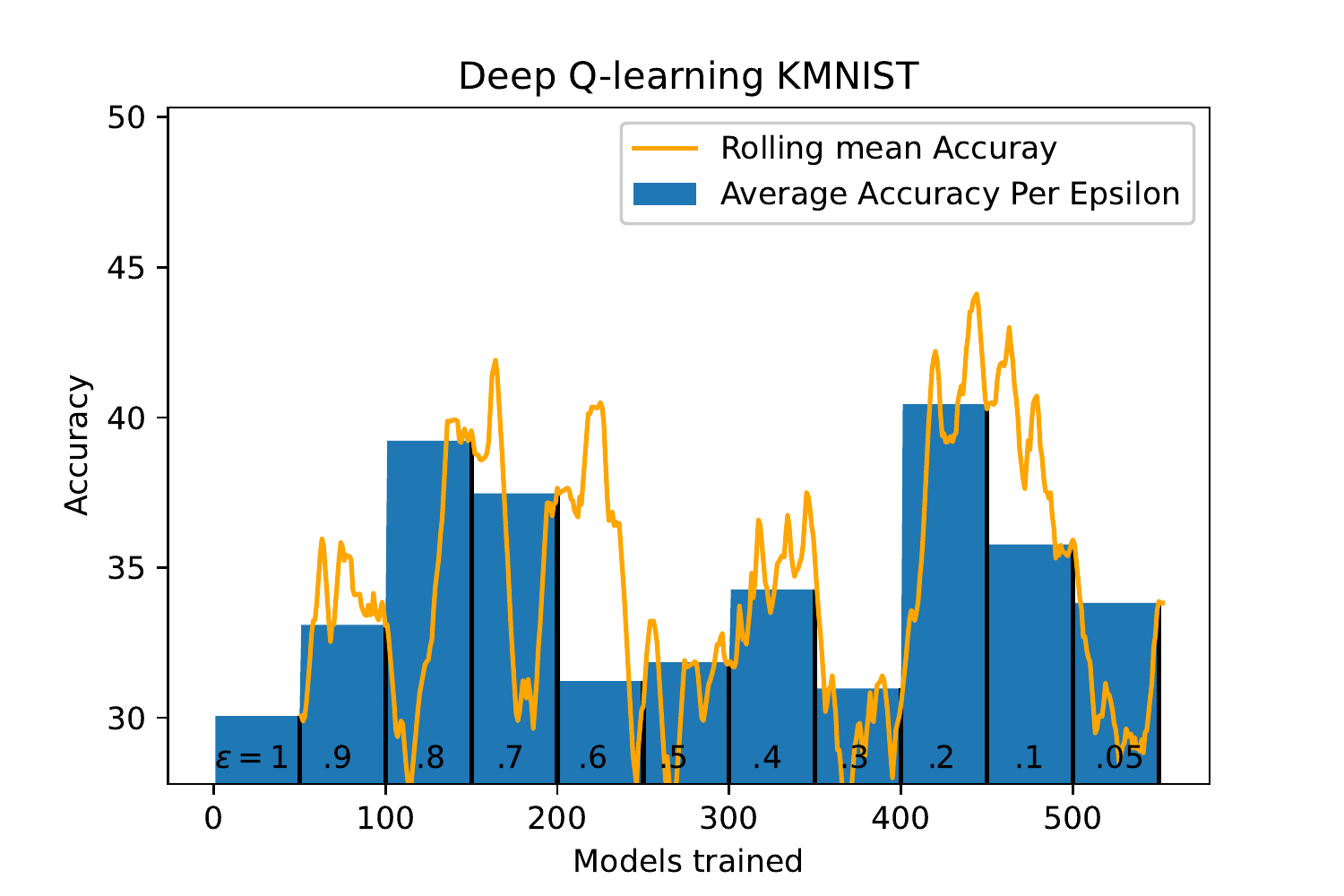} }\label{fig: DQN_KMNIST}}
    \caption{Deep Q-learning performance on several image datasets where the child network is a group convolutional neural network with groups symmetries being any combination of rotation, horizontal symmetry, and vertical symmetry. In addition to group symmetries for network equivariance the Q-learning agent chooses from one or more of six augmentations, namely, random rotation, horizontal flip, vertical flip, translate, scale, and shear. A function of the accuracy is used as the reward function. The Q-learning agent shows improvement in accuracy as $\epsilon$ decreases. Across all the datasets above, we find that equivariance dominates the high-performing networks chosen by the Q-learning over augmentations.}%
    \label{fig: DQN_GCNN}
\end{figure*}
\section{Proofs}\label{sec: proofs}
\begin{proof}[Proof to Thm.~\ref{thm: semidirect_products}]\label{proof: semidirect_products}
($\Rightarrow$) First we show that $\phi$ is equivariant to $G_1$ under $(\mu^{G_1}, \bar{\mu}^{G_1})$ and $G_2$ under $(\mu^{G_2}, \bar{\mu}^{G_2})$ implies $\phi$ is equivariant to $G_1 \rtimes_{\alpha} G_2$ under $(\Gamma, \bar{\Gamma})$. Thus, we want to show that for all $g_1 \in G_1, g_2 \in G_2$, and $x \in \mathcal{X}$, $\phi(\Gamma_{(g_1,g_2)}(x)) = \bar{\Gamma}_{(g_1,g_2)}(\phi(x)),$
assuming $\phi(\mu_{g_i}^{G_i}(x)) = \bar{\mu_{g_i}^{G_i}}(\phi(x))$ for $i \in \{1,2\}$.

Since $(g_1,e)(e,g_2) = (g_1\alpha_{e}(e),eg_2) = (g_1,g_2)$ from the definition of multiplication in semidirect product in Def.~\ref{def: semidirect_product}, and $\Gamma$ is a homomorphism, we have
\begin{equation}\label{eqn:proofthm11}
    \phi(\Gamma_{(g_1,g_2)}(x)) = \phi(\Gamma_{(g_1,e)}(\Gamma_{(e,g_2)}(x))),
\end{equation}
Continuing from (\ref{eqn:proofthm11}) and noting the assumptions in the theorem from (\ref{eqn: thm_1_1}), (\ref{eqn: thm_1_2}), (\ref{eqn: thm_1_3}), (\ref{eqn: thm_1_4}) gives us
\begin{align}
    \phi(\Gamma_{(g_1,e)}(\Gamma_{(e,g_2)}(x))) &= \phi(\mu_{g_1}^{G_1}(\mu_{g_2}^{G_2}(x))) \stackrel{(a)}{=} \bar{\Gamma}_{(g_1,e)}(\phi(\mu_{g_2}^{G_2}(x))) \nonumber\\
    &\stackrel{(b)}{=} \bar{\Gamma}_{(g_1,e)}\bar{\Gamma}_{(e,g_2)}(\phi(x)) \nonumber\\
    &\stackrel{(c)}{=} \bar{\Gamma}_{(g_1,g_2)}(\phi(x)) \nonumber
\end{align}
where (a) holds since $\phi(\mu_{g_1}^{G_1}(\mu_{g_2}^{G_2}(x))) = \bar{\mu}_{g_1}^{G_1}(\phi(\mu_{g_2}^{G_2}(x)))$ because $\phi(\mu_{g_1}^{G_1}(x)) = \bar{\mu}_{g_1}^{G_1}(\phi(x))$, and follows from the assumption $\bar{\Gamma}_{(g_1,e)}(y) = \bar{\mu}_{g_1}^{G_i}(y)$, (b) follows from the equivariance property $\phi(\mu_{g_2}^{G_2}(x)) = \bar{\mu}_{g_2}^{G_2}(\phi(x))$ and from the assumption in the theorem $\bar{\mu}_{g_2}^{G_2}(y) = \bar{\Gamma}_{(e,g_2)}(y)$, and (c) follows since $\bar{\Gamma}$ is a homomorphism and $(g_1,g_2) = (g_1,e)(e,g_2)$. Thus, we have $\phi(\Gamma_{(g_1,g_2)}(x)) = \bar{\Gamma}_{(g_1,g_2)}(\phi(x))$, for all $x \in \mathcal{X}, g_1 \in G_1, g_2 \in G_2$.

Now we prove the other direction of the theorem. Given
\begin{align}\label{eqn: proof_thm_1_7}
    \phi(\Gamma_{(g_1,g_2)}(x)) = \bar{\Gamma}_{(g_1,g_2)}(\phi(x)),
\end{align}
for all $x \in \mathcal{X}, g_1 \in G_1, g_2 \in G_2$, we want to show
$
\phi(\mu_{g_i}^{G_i}(x)) = \bar{\mu}_{g_i}^{G_i}(\phi(x)),
$
for all $x \in \mathcal{X}$, $g_i \in G_i$ for $i \in \{1,2\}$.
For $i=1$, setting $g_2=e, x \in \mathcal{X}$ in \eqref{eqn: proof_thm_1_7}, we have $\phi(\Gamma_{(g_1,e)}(x)) = \bar{\Gamma}_{(g_1,e)}(\phi(x))$, which implies $\phi(\mu_{g_1}^{G_1}(x)) = \bar{\mu}^{G_1}_{g_1}(\phi(x))$ from the assumption $\Gamma_{(g_1,e)}(x) = \mu_{g_1}^{G_1}(x)$, $\bar{\Gamma}_{(g_1,e)}(x) = \bar{\mu}_{g_1}^{G_1}(x)$.
The proof for $i=2$ follows the same procedure and is omitted.
\end{proof}
\begin{proof}[Proof to Claim.~\ref{claim: correctness}]
For simplicity, take $G = G_1 \rtimes_{\alpha} G_2$, i.e. $G_{array} = [G_1,G_2]$, which can be easily extended to general semidirect products. 
From Thm.~\ref{thm: ravanbaksh}, if for each $i \in \mathbf{L}$, if $\Gamma_g(i)$ belongs to the same orbit for all $g \in G$, then $\mathbf{W}$ is $G$-equivariant.

Taking any index $i$, we want to show that all the indices in the set
\begin{align*}
    \Gamma_{G}(i) = \{\Gamma_{(g_x,g_y)}(i): g_x \in G_1, g_y \in G_2\}
\end{align*}
belong to the same orbit. 

First we note that all elements in $Q$ at any point in time in Alg.~\ref{alg: fast-equivariant-network-construction} belong to the same orbit. This is easy to see since for any index $i$ popped out of Q, the visited nodes added to $Q$ are of the form $\Gamma_g(i)$ for some $g \in G$, which belong to the same orbit as $i$. 

Now we show that starting with a non-empty $Q$ consisting of an index $i$, $Q$ becomes empty only when all indices in $\Gamma_{G}(i)$, the orbit in which $i$ belongs, has been visited and assigned the same orbit number, denoted by $C$ in Alg.~\ref{alg: fast-equivariant-network-construction}.
Without loss of generality, say, $i$ is the first index of an orbit popped out of $Q$. 
Then the indices $\Gamma_{g}(i)$ are added to $Q$ and assigned the same orbit for $g\in G_1\cup G_2$. So, thus far we have added $\Gamma_{G_1}(i)\cup \Gamma_{G_2}(i)$ to the orbit and $Q$, where
\begin{align*}
    \Gamma_{G_1}(i) &= \{\Gamma_{g_x,e}(i): g_x \in G_1\}\\
    \Gamma_{G_2}(i) &= \{\Gamma_{e,g_y}(i): g_y \in G_2\}.
\end{align*}
Now, when $\Gamma_{(g_0,e)}(i)$ is popped out of $Q$ for some $g_0 \in G_1$, the indices $\Gamma_{(g_0,g_y)}(i)$ for all $g_y \in G_2$ get added to the orbit. Similarly, when $\Gamma_{(e,g_0)}(i)$ is popped out of $Q$ for some $g_0 \in G_2$, the indices $\Gamma_{(g_x,g_0)}$ for all $g_x \in G_1$ get added to the orbit. Hence, all the indices in the current orbit are visited and assigned the same orbit number.
Further, we only add indices $i$ to $Q$ that were previously not visited, hence $Q$ gets emptied when all the indices in the orbit are visited.

We also know that orbits of a group are disjoint, hence, when $Q$ gets emptied and new indices $i$ with $V[i] < 0$ in $L$ are added to $Q$ with a new orbit, these indices would not overlap with any of the indices already added to other orbits. 
\end{proof}

\section{Deep Q-learning details}\label{sec: fes_details}
\paragraph{Search space}
For Sec.~\ref{subsec: exp_Group_FC_Nets} where we work with fully connected layers, we represent each architecture by a discrete state vector $S$ of length $g_{size}$ and $S[i]$ takes binary values with the network equivariant to the $i$th group if $S[i]=1$, otherwise not.

For Sec.~\ref{subsec: exp_Conv_Nets} where we work with convolutional neural networks, the state $S$ is a vector of length $10$ with the first $3$ binary entries indicating indicating the presence of the groups $P4, H2, V2$ from Tab.~\ref{tab: equivariances_augmentations_GCNN} in that order. The next 6 entries indicate the presence of the 6 augmentations in Tab.~\ref{tab: equivariances_augmentations_GCNN} in that order. The last binary entry is to choose between a large and a small model. We ensure that these networks with small size and different equivariances have the same number of trainable parameters by adjusting the channel numbers in them. For the large network size, we simply multiply the number of channels by a factor of 1.5. This is to check whether a DQN has any particular preference over network sizes.

\paragraph{Action space}
For MLPs and GCNNs, each action is a one-hot vector of length $g_{size}$ and 10 respectively, where if the non-zero index is $i$, then the binary value of $S[i]$ is toggled.

\paragraph{Evaluation strategy}
Since we use the validation accuracy of a trained network as a reward, at each step of the search algorithm, the Q-learning agent trains a network that is equivariant to groups indicated by the state vector by computing the parameter sharing scheme using Alg.~\ref{alg: fast-equivariant-network-construction} for MLPs and directly using precomputed parameter sharing indices of relevant groups from \cite{CohenW2016a} for GCNNs.
However, training a network over the entire dataset at each step is computationally expensive. Hence we use several tricks to speed up our search, inspired by recent works in neural architecture search.

For computing the validation accuracy, we choose small subsets of sizes $4000$ and $1000$ of the dataset as our training and testing sets, similar to \cite{CubukZMVL2019}. Further, to ensure that a model with the same configuration of groups and sizes is not trained twice in the process of search, we store and reuse validation accuracies, as in \cite{BakerGNR2017}. We train each network for 4 epochs and use maximum validation accuracy obtained from after each epoch for computing the reward.
\paragraph{Reward function}
The agent receives a reward proportional to the validation accuracy of the network corresponding to the next state, which it uses to compute $Q$-values in the state-action space. More precisely, we use the reward function
\begin{align}\label{eqn: fes_reward}
    R(s,a) = x\exp{\{|x|\}},
\end{align}
where $x=acc(s,a)-acc_0$, $acc(s,a)$ is the validation accuracy of the network obtained by taking action $a$ at state $s$, $acc_0$ is the accuracy of a baseline MLP without any group equivariance, and $|\cdot|$ is the absolute value. This reward function amplifies both reward for accuracies above baseline accuracies and penalties for accuracies below baseline performance.
\paragraph{Training details} 
For training our deep Q-learning agent, we use the discount factor $\gamma = 0.5$, which allows learning from the rewards obtained from the current state as well as the expected rewards that can be obtained in the future. We use the $\epsilon$-greedy strategy with starting value of $\epsilon=1.0$ and train a fixed number of models per value of $\epsilon$, as shown in Tab.~\ref{tab: epsilon-models-trained}. We use a fully connected network with dimension $12 \times 400 \times 400 \times 400 \times 12$ to compute $Q$-values in our deep Q-learning algorithm. We store the transitions in a replay memory and use a batch size of $512$ and $128$ for training the agent for MLPs and GCNNs respectively to take into account that the number of possible states we considered for GCNNs were relatively fewer.
\begin{table*}
\caption{A fixed number of models were trained per value of $\epsilon$ in our deep Q-learning search for group equivariant MLPs (GEMLPs) and group convolutional neural networks (GCNN).}
    \centering
    \begin{tabular}{|p{2.9cm} p{0.45cm} p{0.45cm} p{0.45cm} p{0.45cm} p{0.45cm} p{0.45cm} p{0.45cm} p{0.45cm} p{0.45cm} p{0.45cm} p{0.45cm} p{0.45cm}|}
    \hline

        \textbf{$\epsilon$ value} & 1.0 & 0.9 & 0.8 & 0.7 & 0.6 & 0.5 & 0.4 & 0.3 & 0.2 & 0.1 & 0.05 & 0.01 \\
        \hline
        \textbf{{\# models} (GEMLP)} & 200 & 100 & 100 & 100 & 100 & 100 & 100 & 50 & 50 & 50 & 50 & --- \\
        \hline
        \textbf{{\# models} (GCNN)} & 50 & 50 & 50 & 50 & 50 & 50 & 50 & 50 & 50 & 50& 50 & 50 \\
        \hline
    \end{tabular}
    \label{tab: epsilon-models-trained}
\end{table*}

\end{document}